\documentclass[11pt]{article}
\usepackage{latexsym}
\usepackage{url}
\usepackage{times}
\usepackage{amsmath}
\usepackage{named}
\usepackage{amsthm}
\input{psfig.sty}

\newcommand{\tuple}[1]{ (#1) }

\newtheorem{proposition}{Proposition}

\newtheorem{lemma}{Lemma}
\newtheorem{corollary}{Corollary}

\def\beq2{\begin{equation}}
\def\eeq2#1{\label{#1}\end{equation}}
\def\sneg{\sim\!}
\def\ii#1{\hbox{\it #1\/}}

\def\no{\ii{not}}

\def\eq{\leftrightarrow}
\def\implies{\rightarrow}
\def\ar{\leftarrow}
\def\o{\overline}
\def\un{\underline}

\author{
Paolo Ferraris\\
Department of Computer Sciences\\
University of Texas at Austin\\
Austin, TX 78712, USA\\
{\tt otto@cs.utexas.edu}
}

\title{Logic Programs with Propositional Connectives and Aggregates}

\author{
Paolo Ferraris\\
Google Inc\\
1600 Anphitheatre Pkwy\\
Mountain View CA 94043, USA\\
{\tt otto@cs.utexas.edu}
}

\begin{document}
\maketitle

\begin{abstract}
Answer set programming (ASP) is a logic programming paradigm that can be used
to solve complex combinatorial search problems. Aggregates are an ASP
construct that plays an important role in many applications.
Defining a satisfactory semantics of aggregates turned out to be a difficult
problem, and in this paper we propose a new approach, based on an analogy
between aggregates and propositional connectives. First, we
extend the definition of an answer set/stable model to cover arbitrary
propositional theories; then we define aggregates on top of them both as
primitive constructs and as abbreviations for formulas. Our definition of
an aggregate combines expressiveness and simplicity, and it inherits many
theorems about programs with nested expressions, such as theorems about
strong equivalence and splitting.
\end{abstract}


\section{Introduction}
Answer set programming (ASP) is a logic programming paradigm that can be used
to solve complex combinatorial search problems~(\cite{mar99}),~(\cite{nie99}).
ASP is based on the stable
model semantics~\cite{gel88} for logic programs: programming in ASP consists
in writing a logic program whose stable models (also called answer sets) 
represent the solution to our problem. ASP has been used, for instance,
in planning~\cite{dim97,lif99c}, model checking~\cite{liu98,hel01},
product configuration~\cite{soi98}, logical cryptanalysis~\cite{hie00},
workflow specification~\cite{tra00,kok01}, reasoning
about policies~\cite{son01}, wire routing problems~\cite{erd00} and
phylogeny reconstruction problems~\cite{erd03a}.

The stable models of a logic program are found by systems called
{\em answer set solvers}. Answer set solvers
can be considered the equivalent of SAT solvers --- systems used to find the
models of propositional formulas --- in logic programming. On the other hand,
it is much easier to express, in logic programming, recursive definitions
(such as reachability in a graph) and defaults.
Several answer set solvers have been developed so far, with
{\sc smodels}\footnote{\tt http://www.tcs.hut.fi/Software/smodels/} and
{\sc dlv}\footnote{\tt http://www.dbai.tuwien.ac.at/proj/dlv/} among the most
popular. As in the case of SAT solvers, answer set solver
competitions --- where answer set solvers are compared to each others in
terms of performance --- are planned to be held
regularly.\footnote{\tt http://asparagus.cs.uni-potsdam.de/contest/}

An important construct in ASP are aggregates.
Aggregates allow, for instance, to perform set operations such as counting
the number of atoms in a set that are true, or summing weights the weights
of the atoms that are true. We can, for instance, express that a node in a
graph has exactly one color by the following cardinality constraint:
$$
1\leq \{c(node,color_1),\dots, c(node,color_m)\} \leq 1.
$$
As another example, a weight constraint of the form
\beq2
3\leq \{p=1, q=2, r=3\}
\eeq2{exweight}
intuitively says that the sum of the weights (the numbers after the ``$=$''
sign) of the atoms from the list $p$, $q$, $r$ that are true is at least 3.

Aggregates are a hot topic in ASP not only because of their importance, but
also because there is no standard understanding of the concept of an aggregate.
In fact, different answer set solvers implement different definitions of
aggregates: for instance, {\sc smodels} implements cardinality and weight
constraints~\cite{nie00}, while {\sc dlv} implements aggregates as defined by
Faber, Leone and Pfeifer~(2005) (we call them FLP-aggregates). Unfortunately,
constructs that are intuitively equivalent to each other may actually
lead to different stable models. In some sense, no current definition of an
aggregate can be considered fully satisfactory, as each of them seems to have
properties that look unintuitive. For instance, it is somehow puzzling that,
as noticed in~\cite{fer05b}, weight constraints
$$
0\leq \{p=2, p=-1\}\qquad\qquad\text{and}\qquad\qquad
0\leq \{p=1\}
$$
are semantically different from each other (may lead to different stable
models). Part of this problem is probably related
to the lack of mathematical tools for studying properties of programs with
aggregates, in particular for reasoning about the correctness of programs
with aggregates.

This paper addresses the problems of aggregates mentioned above by (i)
giving a new semantics of aggregates that, we argue, is more satisfactory than
the existing alternatives, and (ii) providing
tools for studying properties of logic programs with aggregates.

\begin{figure}
    \centering
    \centerline{\psfig{figure=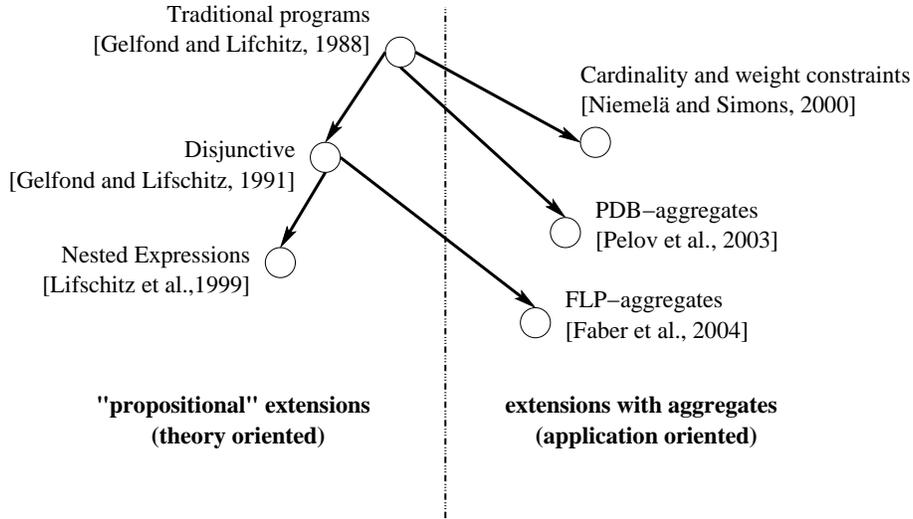,height=7cm}}
    \caption{Evolution of the stable model semantics.}
    \label{fig:evolution}
\end{figure}
Our approach is based on a relationship between two directions of research on
extending the stable model semantics: the work on aggregates, mentioned above,
and the work on  ``propositional extensions'' (see Figure~\ref{fig:evolution}).
The latter makes the syntax of rules more and more similar to the syntax of
propositional formulas. In disjunctive programs,
the head of each rule is a (possibly empty) disjunction of atoms, while
in programs with nested expressions the head and body of each rule can be
any arbitrary formula built with connectives AND, OR and NOT.
For instance,
$$
\neg (p\vee \neg q)\ar p \vee \neg\neg r
$$
is a rule with nested expressions.
Programs with nested expressions are quite attractive especially relative to
point (ii) above, because many theorems about properties of logic programs
have been proved for programs of this kind. For instance, the splitting set
theorem~\cite{lif94e,erdo04} simplifies the task of
computing the stable models of a program/theory by breaking it into two
parts. Work on strong equivalence~\cite{lif01} allows us to modify a
program/theory with the guarantee that stable models are preserved
(more details in Section~\ref{sec:pt-properties}).

Nested expressions have
already been used to express aggregates:~\cite{fer05b} showed
that each weight constraint can be replaced by a nested expressions, 
preserving its stable models. As a consequence, theorems about nested
expressions can be used for programs with weight constraints.
It turns out, however, that nested expressions are not sufficiently general
for defining a semantics for aggregates that overcomes the unintuitive
features of the existing approaches.
For this reason, we extend the syntax of rules with nested
expressions, allowing implication in every part of a ``rule'', and not only
as the outermost connective. (We understand a rule as an implication from
the body to the head). A ``rule'' is then an arbitrary
propositional formula, and a program an arbitrary propositional theory.
Our new definition of a stable model, like all the other definitions,
is based on the process of constructing a reduct. The process that we use
looks very different from all the others, and in particular for programs
with nested expressions. Nevertheless, it turns out
that in application to programs with nested expressions, our definition
is equivalent to the one from~\cite{lif99d}.
This new definition of a stable model also turns out to closely related to
equilibrium logic~\cite{pea97}, a logic based on the concept of a Kripke-model
in the logic of here-and-there. Also, we will show that many theorems about
programs with nested expressions extend to arbitrary propositional theories.

\begin{figure}
    \centering
    \centerline{\psfig{figure=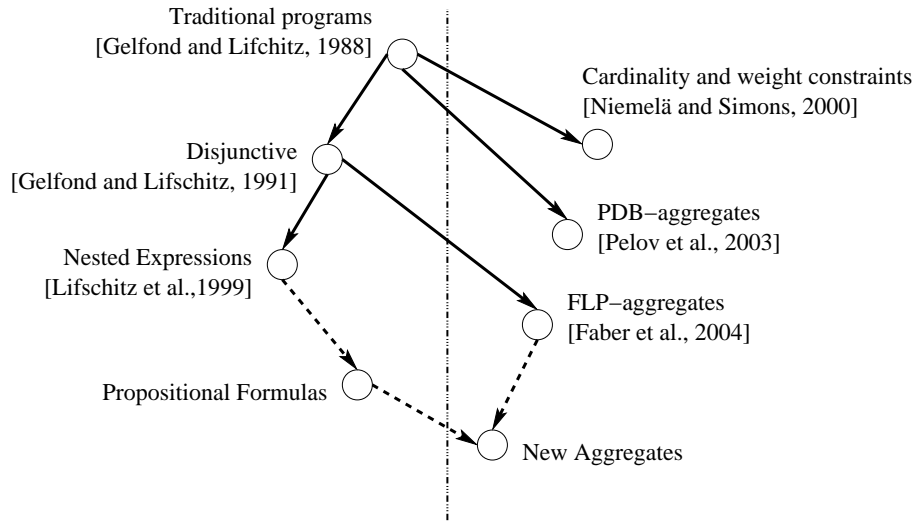,height=7cm}}
    \caption{The proposed extensions}
    \label{fig:proposed}
\end{figure}
On top of arbitrary propositional formulas, we give our definition
of an aggregate. Our extension of the
semantics to aggregates treats aggregates in a way similar to propositional
connectives. Aggregates can be viewed either as primitive constructs or as
abbreviations for propositional formulas; both approaches lead to the same
concept of a stable model. The second view is important because it allows us
to use theorems about stable models of propositional formulas in the
presence of aggregates. As an example of application of such theorems,
we use them to prove the correctness of an ASP program with aggregates
that encodes a combinatorial auction problem.

Syntactically, our aggregates can occur
in any part of a formula, even nested inside each other. (The idea of
``nested aggregates'' is not completely new, as the proof of Theorem~3(a)
in~\cite{fer07c} involves ``nested weight constraints''.)
In our definition of an aggregate we can have, in the same program/theory,
many other kinds of constructs, such as choice rules and disjunction in
the head, while other definitions allow only a subset of them.
Our aggregates seems not to exibit the unintuitive behaviours of other
definitions of aggregates.

It also turns out that a minor syntactical modification of programs with
FLP-aggregates allows us to view them as a special kind of
our aggregates.
(The new picture of extensions is shown in Figure~\ref{fig:proposed}.)
Consequently, we also have a ``propositional'' representation of
FLP-aggregates. We use this fact to compare them with other aggregates
that have a characterization in terms of nested expressions. (As we
said,~\cite{fer05b} showed that weight constraints can be expressed
as nested expressions, and also~\cite{pel03}
implicitly defined PDB-aggregates in terms of nested expressions.)
We will show that all characterizations of aggregates are essentially
equivalent to each other when the aggregates are monotone or antimonotone and
without negation, while there are differences in the other
cases.\footnote{The important role of monotonicity in aggregates has already
been shown, for instance, in~\cite{fab04}.}

The paper is divided into three main parts. We start, in the next section,
with the new definition of a stable model for propositional theories, their
properties and comparisons with previous definitions of stable models and
equilibrium logic. In Section~\ref{sec:aggregates} we
present our aggregates, their properties and the comparisons with other
definitions of aggregates. Section~\ref{sec:proofs} contains all proofs for
the theorems of this paper. The paper ends with the conclusions in
Section~\ref{sec:conclusions}.

Preliminary reports on some results of this paper were published
in~\cite{fer05}.

\section{Stable models of propositional theories}\label{sec:proptheories}

\subsection{Definition}\label{sec:pt-def}

Usually, in logic programming, variables are allowed. As in most definitions
of a stable model, we assume that the variables have been replaced by
constants in a process called ``grounding'' (see, for instance,~\cite{gel88}),
so that we can consider the signature to be essentially propositional.

{\em (Propositional) formulas} are built from atoms and the 0-place
connective $\bot$ (false), using the connectives $\wedge$, $\vee$ and
$\implies$. Even if our definition of a stable model below applies to
formulas with all propositional connectives, we will consider $\top$ as an
abbreviation for $\bot\implies \bot$, a formula $\neg F$ as an abbreviation
for $F\implies \bot$ and $F\eq G$ as an abbreviation for
$(F\implies G) \wedge (G\implies F)$. This will keep notation for other
sections simpler. It can be shown that these abbreviations perfectly capture
the meaning of $\top$, $\neg$ and $\eq$ as primitive connectives
in the stable model semantics.

A {\em (propositional)
theory} is a set of formulas.  As usual in logic programming, truth
assignments will be viewed as sets of atoms; we will write  $X\models F$
to express that a set~$X$ of atoms satisfies a formula~$F$, and similarly
for theories.
 
An implication $F\implies G$ can be also written as a ``rule'' $G\ar F$,
so that traditional programs, disjunctive programs and
programs with nested expressions (reviewed in Section~\ref{sec:ne})
can be seen as special cases of propositional
theories.\footnote{Traditionally, conjunction is represented in a logic
program by a comma, disjunction by a semicolon, and negation as failure
as $\no$.}

We will now define when a  set $X$ of atoms
is a stable model of a propositional theory $\Gamma$. For the rest of the
section $X$ denotes a  set of atoms.

The {\em reduct} $F^X$ of a propositional formula $F$ relative to $X$
is obtained from $F$ by replacing each maximal subformula not satisfied by
$X$ with $\bot$. That is, recursively,
\begin{itemize}
\item
$\bot^X = \bot$;
\item
for every atom $a$, if $X\models a$ then $a^X$ is $a$, otherwise
it is $\bot$; and
\item
for every formulas $F$ and $G$ and any binary connective $\otimes$,
if $X\models F\otimes G$ then $(F\otimes G)^X$ is $F^X\otimes G^X$, otherwise
it is $\bot$.
\end{itemize}
This definition of reduct is similar to a transformation proposed
in~\cite[Section~4.2]{oso04}.

For instance, if $X$ contains $p$ but not $q$ then
\beq2
\begin{split}
(p\ar\neg q)^X&=(p\ar (q\implies \bot))^X=p\ar(\bot\implies\bot)=p\ar\top\\
(q\ar\neg p)^X&=(q\ar(p\implies \bot))^X=\bot\ar\bot\\
((p\implies q)\vee (q\implies p))^X&=\bot\vee (\bot \implies p)\\
\end{split}
\eeq2{exreducts}

The {\em reduct} $\Gamma^X$ of a propositional theory $\Gamma$ relative to
$X$ is $\{F^X:F\in \Gamma\}$. A  set $X$ of atoms is a
{\em stable model} of $\Gamma$ if $X$ is a minimal set satisfying $\Gamma^X$.

For instance, let $\Gamma$ be the theory consisting of
\beq2
\begin{split}
&p\ar \neg q\\
&q\ar\neg p
\end{split}
\eeq2{tradprogram}
Theory $\Gamma$ is actually a traditional program, a logic program in the sense
of~\cite{gel88} (more details in the next section). Set $\{p\}$ is a stable
model of $\Gamma$; indeed, by looking at the first
two lines of~(\ref{exreducts}) we can see that $\Gamma^{\{p\}}$ is
$\{p\ar\top,\bot\ar\bot\}$, which is satisfied by $\{p\}$ but not by its
unique proper subset $\emptyset$.
It is easy to verify that $\{q\}$ is the only other
stable model of $\Gamma$. Similarly, it is not difficult to see that
$\{p\}$ is the only stable model of the theory
\beq2
\begin{split}
&(p\implies q)\vee (q\implies p)\\
&p
\end{split}
\eeq2{proptheory}
(The reduct relative to $\{p\}$ is $\{\bot\vee (\bot \implies p),p\}$).

As the name suggests, a stable model of a propositional theory
$\Gamma$ is a model --- in the sense of classical logic --- of $\Gamma$.
Indeed, it follows from the easily verifiable
fact that, for each  set $X$ of atoms, $X\models \Gamma^X$ iff
$X\models \Gamma$. On the other hand,
formulas that are equivalent in classical logic may have different stable
models: for instance,
$\{\neg\neg p\}$ has no stable models, while $\{p\}$ has stable model $\{p\}$.
Proposition~\ref{prop:se} below will give some characterizations of
transformations that preserves stable models.
Notice that classically equivalent transformations can be applied to the
reduct of a theory, as the sets of atoms that are minimal don't change.

Finally, a note about a second kind of negation in propositional theories.
In \cite[Section~3.9]{fer05e}, atoms were divided into two groups: ``positive''
and ``negative'', so that each negative atom has the form $\sneg a$, where $a$
is a positive atom.  Symbol $\sneg\ $ is called ``strong negation'', to
distinguish it from the connective $\neg$, which is called
{\em negation as failure}.\footnote{Strong negation was introduced in the
syntax of logic programs in~\cite{gel91b}. In that paper, it was called
``classical negation'' and treated not as a part of an atom, but rather as a
logical operator.} In presence of strong negation, the stable model semantics
says that only sets of atoms that don't contain both atoms $a$ and $\sneg a$
can be stable models. For simplicity, we will make no distinctions between
positive and negative atoms, considering that we can remove the sets of atoms
containing any pair of atoms $a$ and $b$ from the stable models of a theory
by adding a formula $\neg (a\wedge b)$ to the theory.
(See Proposition~\ref{prop:constraint}).

\subsection{Relationship with previous definitions of a stable model}
\label{sec:ne}

As mentioned in the introduction, a propositional theory is the extension
of traditional programs~\cite{gel88}, disjunctive programs~\cite{gel91b}
and programs with nested expressions~\cite{lif99d}
(see Figure~\ref{fig:proposed}). We want to compare the
definition of a stable model from the previous section with the definitions
in the three papers cited above.

\begin{figure}
\begin{tabular}{|c|c|}
\hline
kind of rule& syntax \\
\hline
\hline
traditional& $a\ar l_1\wedge\cdots\wedge l_n$\\
\hline
disjunctive& $a_1\vee\cdots\vee a_m\ar l_1\wedge\cdots\wedge l_n$\\
\hline
with nested expressions& $F\ar G\quad$ ($F$ and $G$ are nested expressions)\\
\hline
\end{tabular}
\caption{Syntax of ``propositional'' logic programs. Each $a,a_1,\dots,a_m$
($m\geq 0$) denotes an atom, and each $l_1,\dots,l_n$ ($n\geq 0$) a literal
--- an atom possibly prefixed by $\neg$. A {\em nested expression} is
any formula that contains no implications other than negations or $\top$.}
\label{fig:ne}
\end{figure}

The syntax of a {\em traditional rule}, {\em disjunctive rule} and
{\em rule with nested expressions} are shown in Figure~\ref{fig:ne}.
We understand an empty conjunction as $\top$ and an empty disjunction as
$\bot$, so that traditional and disjunctive rules are also rules with
nested expressions. The part before and after the arrow $\ar$ are called the
{\em head} and the {\em body} of the rule, respectively. When the body
is empty (or $\top$), we can denote the whole rule by its head.
A {\em logic program} is a set of rules. If all rules in a logic program are
traditional then we
say that the program is {\em traditional} too, and similarly for the other
two kinds of rules.

For instance,~(\ref{tradprogram}) is a traditional program as well as
a disjunctive program and a program with nested expressions. On the
other hand,~(\ref{proptheory}) is not a logic program of any of those kinds,
because of the first formula that contains implications nested in a
disjunction.

For all kinds of programs described above, the definition of a stable model
is similar to ours for propositional theories: to check whether a 
set $X$ of atoms is a stable model of a program $\Pi$, we (i) compute the
reduct of $\Pi$ relative to $X$, and (ii) verify if $X$ is a minimal model
of such reduct. On the other hand, the way in which the reduct is computed
is different. We consider the definition from~\cite{lif99d}, as the
definitions from~\cite[\citeyear{gel91b}]{gel88} are essentially its
special cases.

The {\em reduct} $\Pi^{\un X}$ of a program $\Pi$ with nested expressions
relative to a  set $X$ of atoms is the result of replacing,
in each rule of $\Pi$, each maximal subformula of the form $\neg F$
with $\top$ if $X\models \neg F$, and with $\bot$ otherwise.
Set $X$ is a {\em stable
model} of $\Pi$ if it is a minimal model of $\Pi^{\un X}$.~\footnote{We
underline the set $X$ in $\Pi^{\un X}$ to distinguish this
definition of a reduct from the one from the previous section.}

For instance, if $\Pi$ is~(\ref{tradprogram}) then the reduct
$\Pi^{\un {\{p\}}}$ is
\begin{equation*}
\begin{split}
&p\ar \top\\
&q\ar \bot,\\
\end{split}
\end{equation*}
while $\Pi^{\un \emptyset}$ is
\begin{equation*}
\begin{split}
&p\ar \top\\
&q\ar \top,\\
\end{split}
\end{equation*}
The stable models of $\Pi$ --- based on this definition of the reduct ---
are the same ones that we computed in the previous section using the newer
definition of a reduct: $\{p\}$ and $\{q\}$. On the
other hand, there are differences in the value of the reducts: for instance,
we have just seen that $\Pi^{\un \emptyset}$ is classically equivalent to
$\{p,q\}$, while $\Pi^\emptyset=\{\bot,\bot\}$. However, some similarities
between these definitions exist. For instance, negations are treated
essentially in the same way: a nested expression $\neg F$ is transformed into
$\bot$ if $X\models F$, and into
$\top$ otherwise, under both definitions of a reduct.

The following proposition states a more general relationship between the new
definition and the 1999 definition of a reduct.

\begin{proposition}
\label{pt-prop1}
For any program $\Pi$ with nested expressions and any set $X$ of atoms,
$\Pi^X$ is equivalent, in the sense of classical logic,
\begin{itemize}
\item
to $\bot$, if $X\not\models \Pi$, and
\item
to the program obtained from $\Pi^{\un X}$ by replacing all atoms
that do not belong to $X$ by $\bot$, otherwise.
\end{itemize}
\end{proposition}

\begin{corollary}
\label{cor1}
Given two sets of atoms $X$ and $Y$ with $Y\subseteq X$ and any program
$\Pi$ with nested expressions,
$Y\models \Pi^X$ iff $X\models \Pi$ and $Y\models \Pi^{\un X}$.
\end{corollary}

From the corollary above, one of the main claims of this paper follows,
that our definition of a stable model is an extension of the definition
for programs with nested expressions.

\begin{proposition}
\label{th:th1}
For any program $\Pi$ with nested expressions, the collections of stable
models of $\Pi$ according to our definition and according to~\cite{lif99d}
are identical.
\end{proposition}

\subsection{Relationship with Equilibrium Logic}

Equilibrium logic~\cite[\citeyear{pea99}]{pea97} is defined in terms
of Kripke models in the logic of here-and-there, a logic intermediate between
intuitionistic and classical logic.

The logic of here-and-there is a 3-valued logic, where an interpretation
(called an {\em HT-interpretation}) is represented by a pair $(X,Y)$ of
sets of atoms where $X\subseteq Y$. Intuitively, atoms in $X$ are considered
``true'', atoms not in $Y$ are considered ``false'', and all other atoms
(that belong to $Y$ but not $X$) are ``undefined''.

An HT-interpretation $(X,Y)$ {\em satisfies} a formula $F$ (symbolically,
$(X,Y)\models F$) based on the following recursive definition ($a$ stands
for an atom):
\begin{itemize}
\item
$(X,Y)\models a$ iff $a\in X$,
\item
$(X,Y)\not\models \bot$,
\item
$(X,Y)\models F\wedge G$ iff $(X,Y)\models F$ and $(X,Y)\models G$,
\item
$(X,Y)\models F\vee G$ iff $(X,Y)\models F$ or $(X,Y)\models G$,
\item
$(X,Y)\models F\implies G$ iff $(X,Y)\models F$ implies $(X,Y)\models G$,
and $Y$ satisfies $F\implies G$ in classical logic.
\end{itemize}
An HT-interpretation $(X,Y)$ {\em satisfies} a propositional theory if it
satisfies all the elements of the theory. Two formulas are {\em equivalent} in
the logic of here-and-there if they are satisfied by the same
HT-interpretations.

Equilibrium logic defines when a  set $X$ of atoms
is an equilibrium model of a propositional theory $\Gamma$. Set
$X$ is {\em an equilibrium model} of $\Gamma$ if $(X,X)\models \Gamma$ and,
for all proper subsets $Z$ of $X$, $(Z,X)\not\models \Gamma$.

A relationship between the concept of a model in the logic of here-and-there,
and satisfaction of the reduct exists.

\begin{proposition}
\label{lemma:mainlemma}
For any formula $F$ and any HT-interpretation $(X,Y)$, $(X,Y)\models F$ iff
$X\models F^Y$.
\end{proposition}

Next proposition compares the concept of an equilibrium model with the new
definition of a stable model.

\begin{proposition}
\label{pt-mainth}
For any theory, its models in the sense of equilibrium logic are identical
to its stable models.
\end{proposition}

This proposition offers another way of proving Proposition~\ref{th:th1},
as~\cite{lif01} showed that the equilibrium models of a program with nested
expressions are the stable models of the same program in the sense
of~\cite{lif99d}.

\subsection{Properties of propositional theories}\label{sec:pt-properties}

This section shows how several theorems about logic programs with nested
expressions can be extended to propositional theories.

\subsubsection{Strong equivalence}

Two theories $\Gamma_1$ and $\Gamma_2$ are {\em strongly equivalent} if,
for every theory $\Gamma$, $\Gamma_1\cup \Gamma$ and $\Gamma_2\cup \Gamma$
have the same stable models.

\begin{proposition}\label{prop:se}
For any two theories $\Gamma_1$ and $\Gamma_2$, the following conditions
are equivalent:
\begin{enumerate}
\item[(i)]
$\Gamma_1$ is strongly equivalent to $\Gamma_2$,
\item[(ii)]
$\Gamma_1$ is equivalent to $\Gamma_2$ in the logic of here-and-there, and
\item[(iii)]
for each  set $X$ of atoms, $\Gamma_1^X$ is equivalent to $\Gamma_2^X$
in classical logic.
\end{enumerate}
\end{proposition}

The equivalence between~(i) and~(ii) is essentially Lemma~4
from~\cite{lif01} about equilibrium logic. The equivalence between~(i)
and~(iii) is similar to Theorem~1 from~\cite{tur03} about nested expressions,
but simpler and more general. Notice that~(iii) cannot be replaced by
\begin{enumerate}
\item[(iii')]
for each  set $X$ of atoms, $\Gamma_1^{\un X}$ is equivalent to
$\Gamma_2^{\un X}$ in classical logic,
\end{enumerate}
not even when $\Gamma_1$ and $\Gamma_2$ are programs with nested expressions.
Indeed, $\{p\ar\neg p\}$ is strongly equivalent to $\{\bot\ar\neg p\}$, but
$\{p\ar\neg p\}^{\un \emptyset}=\{p\ar \top\}$ is not classically
equivalent to $\{\bot\ar\neg p\}^{\un \emptyset}=\{\bot\ar \top\}$.

Replacing, in a theory $\Gamma$, a (sub)formula $F$ with a formula $G$ is
guaranteed to preserve strong equivalence iff $F$ is strongly equivalent to
$G$. Indeed, strong equivalence between $F$ and $G$ is clearly a necessary
condition: take $\Gamma=\{F\}$. It is also sufficient because --- as in
classical logic --- replacements of formulas with equivalent formulas in
the logic of here-and-there preserves equivalence in the same logic.

Cabalar and Ferraris~[\citeyear{cab04}] showed that any propositional theory
is strongly
equivalent to a logic program with nested expressions. That is, a propositional
theory can be seen as a different way of writing a logic program.
This shows that the concept of a stable model for propositional theories
is not too different from the concept of a stable model for a logic program.

\subsubsection{Other properties}

To state several propositions below, we need the following definitions.
Recall that an expression of the form $\neg F$ is an abbreviation
for $F\implies \bot$, and equivalences are the conjunction of two opposite
implications.
An occurrence of an atom in a formula is {\em positive} if it is in
the antecedent of an even number of implications. An occurrence is
{\em strictly positive} if such number is 0, and {\em negative} if it
odd.\footnote{The concept of a positive and negative occurrence of an
atom should not be confused by the concept of a ``positive'' and ``negative''
atom mentioned at the end of Section~\ref{sec:pt-def}.}
For instance, in a formula
$(p\implies r)\implies q$, the occurrences of $p$ and $q$ are
positive, the one of $r$ is negative, and the one of $q$ is strictly positive.

The following proposition is an extension of the property that in each
stable model of a program, each atom occurs in the head of a rule of that
program~\cite[Section~3.1]{lif96b}.
An atom is an {\em head atom} of a theory $\Gamma$ if it has a
strictly positive occurrence in $\Gamma$.~\footnote{In case of
programs with nested expressions, it is easy to check that head atoms are
atoms that occur in the head of a rule outside the scope of negation $\neg$.}

\begin{proposition}\label{prop:head}
Each stable model of a theory $\Gamma$ consists of head atoms of $\Gamma$.
\end{proposition}

A rule is called a {\em constraint} if its head is $\bot$.
In a logic program, adding constraints
to a program $\Pi$ removes the stable models of $\Pi$ that don't satisfy the
constraints. A constraint can be seen as a formula of the form $\neg F$,
a formula that doesn't have head atoms. Next proposition generalizes the
property of logic programs stated above to propositional theories.
\begin{proposition}\label{prop:constraint}
For every two propositional theories $\Gamma_1$ and $\Gamma_2$ such that
$\Gamma_2$ has no head atoms, a  set $X$ of atoms
is a stable model of $\Gamma_1\cup \Gamma_2$ iff $X$ is a stable model of
$\Gamma_1$ and $X\models \Gamma_2$.
\end{proposition}

The following two propositions are generalizations of
propositions stated in~\cite{fer05b} in the case of
logic programs. We say that an occurrence of an atom is {\em in the scope
of negation} when it occurs in a formula $\neg F$.

\begin{proposition}[Lemma on Explicit Definitions]\label{prop:expl}
Let $\Gamma$ be any propositional theory, and $Q$ a set of atoms
not occurring in $\Gamma$. For each $q\in Q$, let $Def(q)$ be a
formula that doesn't contain any atoms from $Q$. Then
$X\mapsto X\setminus Q$ is a 1--1 correspondence between the stable models of
$\Gamma\cup \{Def(q)\implies q: q\in Q\}$ and the stable models of~$\Gamma$.
\end{proposition}

\begin{proposition}[Completion Lemma]\label{prop:compl}
Let $\Gamma$ be any propositional theory, and $Q$ a set of atoms
that have positive occurrences in $\Gamma$ only in the scope of negation.
For each $q\in Q$, let $Def(q)$ be a formula such that all negative
occurrences of atoms from $Q$ in $Def(q)$ are in the scope
of negation. Then $\Gamma\cup \{Def(q)\implies q: q\in Q\}$ and
$\Gamma\cup \{Def(q)\eq q : q\in Q\}$ have the same stable models.
\end{proposition}

The following proposition is essentially a generalization of the splitting
set theorem from~\cite{lif94e} and~\cite{erdo04}, which allows to break
logic programs/propositional theories into parts and compute the
stable models separately. A formulation of this theorem has also been stated
in~\cite{fer05e} in the special case of theories consisting
of a single formula.

\begin{proposition}[Splitting Set Theorem]\label{prop:split}
Let $\Gamma_1$ and $\Gamma_2$ be two theories such that no atom
occurring in $\Gamma_1$ is a head atom of $\Gamma_2$. Let $S$ be
a set of atoms containing all head atoms of $\Gamma_1$ but no head atoms
of $\Gamma_2$.
A  set $X$ of atoms is a stable model of $\Gamma_1\cup \Gamma_2$
iff $X\cap S$ is a stable model of $\Gamma_1$ and $X$ is a stable model
of $(X\cap S)\cup\Gamma_2$.
\end{proposition}

\subsection{Computational complexity}\label{sec:prop-compl}

Since the concept of a stable model is equivalent to the concept of an
equilibrium model, checking the existence of a stable model of a
propositional theory is a $\Sigma_2^P$-complete problem as for
equilibrium models~\cite{pea01}. Notice that the existence of a stable model
of a disjunctive program is already
$\Sigma_2^P$-hard~\cite[Corollary~3.8]{eit93a}.

The existence of a stable model for a traditional program is a  
NP-complete problem~\cite{mar91}. The same holds, more generally, for
logic programs with nested expressions where the head of each rule is an
atom or $\bot$. (We call programs of this kind {\em nondisjunctive}).
We may wonder if the same property holds for arbitrary sets of formulas of
the form $F\implies a$ and $F\implies \bot$. The answer is negative:
the following lemma shows that as soon as we allow implications in
formulas $F$ then we have the same expressivity --- and then complexity --- as
disjunctive rules.

\begin{lemma}
\label{lemma:disjunction}
Rule
$$
l_1\wedge\cdots\wedge l_m \implies a_1\vee\cdots\vee a_n
$$
($n>0, m\geq 0$) where $a_1,\dots,a_n$ are atoms and $l_1,\dots,l_m$ are
literals, is strongly equivalent to the set of $n$ implications $(i=1,\dots,n)$
\beq2
(l_1\wedge\cdots\wedge l_m\wedge 
(a_1\implies a_i)\wedge \cdots \wedge (a_n\implies a_i)) \implies a_i.
\eeq2{translaggr}
\end{lemma}

\begin{proposition}
\label{prop:disjunction}
The problem of the existence of a stable model of a theory consisting of
formulas of the form $F\implies a$ and $F\implies \bot$\ is
$\Sigma_2^P$-complete.
\end{proposition}

We will see, in Section~\ref{sec:aggr-compl}, that the conjunctive terms
in the antecedent of~(\ref{translaggr}) can equivalently be replaced by
aggregates of a simple kind, thus showing that allowing aggregates in
nondisjunctive programs increases their computational complexity.

\section{Aggregates}\label{sec:aggregates}

\subsection{Syntax and semantics}\label{sec:defaggr}

A {\em formula with aggregates} is defined recursively as follows:
\begin{itemize}
\item atoms and $\bot$ are
formulas with aggregates\footnote{Recall that $\top$ is an abbreviation for
$\bot\implies \bot$},
\item
propositional combinations of formulas with aggregates
are formulas with aggregates, and
\item
any expression of the form 
\beq2
op\langle \{F_1=w_1,\dots, F_n=w_n\}\rangle\prec N
\eeq2{aggregate}
where
\begin{itemize}
\item
$op$ is (a symbol for) a function from multisets of real numbers to
$\mathcal{R}\cup \{-\infty,+\infty\}$ (such as sum,
product, min, max, etc.),
\item
$F_1,\dots,F_n$ are formulas with aggregates, and $w_1,\dots,w_n$ are
(symbols for) real numbers (``weights''),
\item
$\prec$ is (a symbol for) a binary relation between real numbers, such
as $\leq$ and $=$, and
\item
$N$ is (a symbol for) a real number,
\end{itemize}
is a formula with aggregates.
\end{itemize}
A {\em theory with aggregates} is a set of formulas with aggregates.
A formula of the form~(\ref{aggregate}) is called an {\em aggregate}.

The intuitive meaning of an aggregate is explained by the following clause,
which extends the definition of satisfaction of propositional formulas to
arbitrary formulas with aggregates.
For any aggregate~(\ref{aggregate}) and any set $X$ of atoms, let $W_X$ be
the multiset $W$ consisting of
the weights $w_i$ ($1\leq i\leq n$) such that $X\models F_i$; we say that
$X$ {\em satisfies}~(\ref{aggregate})
if $op(W_X)\prec N$.
For instance,
\beq2
sum\langle \{p=1,q=1\} \rangle\not = 1
\eeq2{exaggr}
is satisfied by the  sets of atoms that satisfy both $p$ and $q$ or
none of them.

As usual, we say that $X$
satisfies a theory $\Gamma$ with aggregates if
$X$ satisfies all formulas in $\Gamma$. We extend the concept of classical
equivalence to formulas/theories with aggregates.

We extend the definition of a stable models of propositional theories
(Section~\ref{sec:proptheories}) to
cover aggregates, in a very natural way. Let $X$ be a  set of atoms.
The {\em reduct} $F^X$ of a formula $F$ with aggregates relative to $X$ is
again the result of replacing each maximal formula not satisfied by $X$ with
$\bot$. That is, it is sufficient to add a clause relative to aggregates to the
recursive definition of a reduct:
for an aggregate $A$ of the form~(\ref{aggregate}),
$$
A^X=\begin{cases}
op\langle \{F_1^X=w_1,\dots, F_n^X=w_n\}\rangle\prec N,&\text{if $X\models A$},\cr
\bot,&\text{otherwise.}\\
\end{cases}
$$

This is similar to the clause for binary connectives:
$$
(F\otimes G)^X=\begin{cases}
F^X\otimes G^X, &\text{if $X\models F\otimes G$},\cr
\bot,&\text{otherwise.}\\
\end{cases}
$$

The rest of the definition of a stable model remains the same: the {\em reduct} $\Gamma^X$ of a theory $\Gamma$ with aggregates is $\{F^X: F\in \Gamma\}$,
and $X$ is a {\em stable model} of $\Gamma$ if $X$ is a minimal model of
$\Gamma^X$.

Consider, for instance, the theory $\Gamma$ consisting of one formula
\beq2
sum\langle \{p=-1,q=1\}\rangle\geq 0\implies q.
\eeq2{exaggrsem}
Set $\{q\}$ is a stable model of $\Gamma$. Indeed, since both the
antecedent and consequent of~(\ref{exaggrsem}) are satisfied by $\{q\}$,
$\Gamma^{\{q\}}$ is
$$
sum\langle \{\bot=-1,q=1\}\rangle\geq 0\implies q.
$$
The antecedent of the implication above is satisfied by every set of atoms,
so the whole formula
is equivalent to $q$. Consequently, $\{q\}$ is the minimal model of
$\Gamma^{\{q\}}$, and then a stable model of $\Gamma$.

\subsection{Aggregates as Propositional Formulas}\label{sec:pt-pf}

A formula/theory with aggregates can also be seen as a normal propositional
formula/theory, by 
identifying~(\ref{aggregate}) with the formula
\begin{equation}\label{mainaggr}
\bigwedge_{I\subseteq \{1,\dots,n\}~:~op(\{w_i~:~i\in I\})\not\prec N}
 \big( \big( \bigwedge_{i\in I} F_i \big) \implies
\big( \bigvee_{i\in \o I} F_i \big) \big),
\end{equation}
where $\o I$ stands for $\{1,\dots, n\}\setminus I$,
and $\not\prec$ is the negation of $\prec$.

For instance, if we consider aggregate~(\ref{exaggr}), the
conjunctive terms in~(\ref{mainaggr}) correspond to the
cases when the sum of weights is 1, that is, when $I=\{1\}$ and $I=\{2\}$.
The two implications are
$q\implies p$ and $p\implies q$ respectively, so that~(\ref{exaggr}) is
\beq2
(q\implies p)\wedge (p\implies q).
\eeq2{exaggr2}
Similarly,
\beq2
sum\langle \{p=1,q=1\} \rangle = 1
\eeq2{exaggrno}
is
\beq2
(p\vee q)\wedge \neg (p\wedge q).
\eeq2{exaggrno2}

Even though~(\ref{exaggrno}) can be seen as the negation of~(\ref{exaggr}),
the negation of~(\ref{exaggrno2}) is not strongly equivalent 
to~(\ref{exaggr2}) (although they are classically
equivalent). This shows that it is generally incorrect to ``move'' a
negation from a binary relation symbol (such as $\not =$) in front of
the aggregate as the unary connective $\neg$, and vice versa.

Next proposition shows that this understanding of aggregates as propositional
formulas is equivalent to the semantics for theories with aggregates of the
previous section. Two formulas with aggregates are {\em classically equivalent}
to each other if they are satisfied by the same sets of atoms.

\begin{proposition}
\label{propaggregate}
Let~$A$ be an aggregate of the form~(\ref{aggregate}) and let $G$ be the
corresponding formula~(\ref{mainaggr}). Then
\begin{enumerate}
\item[(a)]
$G$ is classically equivalent to $A$, and
\item[(b)]
for any set $X$ of atoms,
$G^X$ is classically equivalent to $A^X$.
\end{enumerate}
\end{proposition}

Treating aggregates as propositional formulas allows us to apply many
properties of propositional theories presented in
Section~\ref{sec:pt-properties} to theories with aggregates also.
Consequently, we have the concept of an head atom, of strong
equivalence, we can use the completion lemma and so on. We will use several
of those properties to prove Proposition~\ref{exprop2} below.
In the rest of the paper we will often make no distinctions
between the two ways of defining the semantics of aggregates discussed here.

Notice that replacing, in a theory, an aggregate of the
form~(\ref{aggregate}) with a formula that is not strongly equivalent to the
corresponding formula~(\ref{mainaggr}) may lead to different stable models.
This shows that there is no other way (modulo strong equivalence) of
representing our aggregates as propositional formulas.

\subsection{Monotone Aggregates}\label{sec:monaggr}

An aggregate $op\langle  \{F_1=w_1,\dots, F_n=w_n\}\rangle\prec N$ is
{\em monotone} if, for each pair of multisets $W_1$, $W_2$ such that
$W_1\subseteq W_2\subseteq \{w_1,\dots,w_n\}$, $op(W_2)\prec N$ is true
whenever $op(W_1)\prec N$ is true. The definition of an {\em antimonotone}
aggregate is similar, with $W_1\subseteq W_2$ replaced by $W_2\subseteq W_1$.

For instance,
\beq2
sum\langle \{ p=1,q=1 \}\rangle> 1
\eeq2{exaggrmon}
is monotone, and 
\beq2
sum\langle \{ p=1,q=1 \}\rangle< 1.
\eeq2{exaggrantimon}
is antimonotone. An example of an aggregate that is neither monotone nor
antimonotone is~(\ref{exaggr}).

\begin{proposition}
\label{monotone}
For any aggregate $op\langle  \{F_1=w_1,\dots, F_n=w_n\}\rangle\prec N$,
formula~(\ref{mainaggr}) is strongly equivalent to
\beq2
\bigwedge_{I\subseteq \{1,\dots,n\}~:~op(\{w_i~:~i\in I\})\not\prec N}
 \big( \bigvee_{i\in \o I} F_i \big)
\eeq2{monaggr}
if the aggregate is monotone, and to
\beq2
\bigwedge_{I\subseteq \{1,\dots,n\}~:~op(\{w_i~:~i\in I\})\not\prec N}
 \big( \neg \bigwedge_{i\in I} F_i \big)
\eeq2{antimonaggr}
if the aggregate is antimonotone.
\end{proposition}
In other words, if $op\langle  S\rangle\prec N$ is monotone then the
antecedents of the implications in~(\ref{mainaggr}) can be dropped.
Similarly, in case of antimonotone aggregates, the consequents of these
implications can be replaced by $\bot$. In both cases,~(\ref{mainaggr})
is turned into a nested expression, if $F_1,\dots,F_n$ are nested
expressions.

For instance, aggregate~(\ref{exaggrmon}) is normally written as formula
$$(p\vee q)\wedge (p\implies q)\wedge (q\implies p).$$
Since the aggregate is monotone, it can also be written, by
Proposition~\ref{monotone}, as nested expression
$$(p\vee q)\wedge q\wedge p,$$
which is strongly equivalent to $q\wedge p$.
Similarly, aggregate~(\ref{exaggrantimon}) is normally written as formula
$$((p\wedge q)\implies\bot)\wedge (p\implies q)\wedge (q\implies p);$$
since the aggregate is nonmonotone, it can also be written as nested expression
$$\neg (p\wedge q)\wedge \neg p\wedge \neg q,$$
which is strongly equivalent to $\neg p\wedge \neg q$.

On the other hand, if an aggregate is neither monotone nor antimonotone,
it may be not possible to find a nested expression strongly equivalent
to~(\ref{mainaggr}), even if $F_1,\dots,F_n$ are nested
expressions. This is the case for~(\ref{exaggr}).
Indeed, the formula~(\ref{mainaggr}) corresponding to~(\ref{exaggr})
is~(\ref{exaggr2}), whose reduct relative to $\{p,q\}$ is~(\ref{exaggr2}).
Consequently, by Proposition~\ref{prop:se}, for any formula $G$
strongly equivalent to~(\ref{exaggr2}),
$G^{\{p,q\}}$ is classically equivalent to~(\ref{exaggr2}).
On the other hand, the reduct of nested expressions are essentially
AND-OR combinations of
atoms, $\top$ and $\bot$ (negations either become $\bot$ or $\top$ in the
reduct), and no formula of this kind is classically equivalent
to~(\ref{exaggr2}).

In some uses of ASP, aggregates that are neither monotone nor antimonotone
are essential, as discussed in the next section.



\subsection{Example}\label{sec:example}

We consider the following variation of the combinatorial auction
problem~\cite{bar01},
which can be naturally formalized using an aggregate that is neither monotone
nor antimonotone.

Joe wants to move to another town and has the problem of removing all his
bulky furniture from his old place. He has received some bids: 
each bid may be for one piece or several pieces of furniture, and the amount
offered can be negative (if the value of the pieces is lower than the cost of
removing them). A junkyard will take any object not sold to bidders, for a
price. The goal is to find a collection of bids for which Joe doesn't lose
money, if there is any.

Assume that there are $n$ bids, denoted by atoms $b_1,\dots,b_n$.
We express by the formulas
\beq2
b_i\vee \neg b_i
\eeq2{joe1}
($1\leq i\leq n$) that Joe is free to accept any bid or not.
Clearly, Joe cannot accept two bids that involve the selling of the same
piece of furniture. So, for every such pair $i,j$ of bids, we include the
formula
\beq2
\neg (b_i\wedge b_j).
\eeq2{joe2}
Next, we need to express which pieces
of the furniture have not been given to bidders. If there are $m$
objects we can express that an object $i$ is sold by bid $j$ by adding
the rule
\beq2
b_j\implies s_i
\eeq2{joe3}
to our theory.

Finally, we need to express that Joe doesn't lose money by selling his items.
This is done by the aggregate
\beq2
sum\langle \{b_1=w_1,\dots,b_n=w_n,
\neg s_1=-c_1,\dots,\neg s_m=-c_m\}\rangle\geq 0,
\eeq2{joe4}
where each $w_i$ is the amount of money (possibly negative) obtained by
accepting bid $i$, and each $c_i$ is the money requested by the junkyard
to remove item $i$. Note that~(\ref{joe4})
is neither monotone nor antimonotone.

We define a {\em solution} to Joe's problem as a set of accepted bids such that
\begin{enumerate}
\item[(a)]
the bids involve selling disjoint sets of items, and
\item[(b)]
the sum of the money earned from the bids is greater than the money spent
giving away the remaining items.
\end{enumerate}

\begin{proposition}\label{exprop2}
$X\mapsto \{i: b_i\in X\}$ is a 1--1 correspondence between the
stable models of the theory consisting of formulas~(\ref{joe1})--(\ref{joe4})
and a solution to Joe's problem.
\end{proposition}

\subsection{Computational Complexity}\label{sec:aggr-compl}

Since theories with aggregates generalize disjunctive programs, the
problem of the existence of a stable model of a theory with aggregates
clearly is $\Sigma_2^P$-hard.\footnote{We are clearly assuming weight not to
be arbitrary real numbers but to belong to a countable subset of real numbers,
such as integers of floating point numbers.} We need to check in which class
of the computational hierarchy this problem belongs.

Even if propositional formulas corresponding to aggregates can be
exponentially larger than the original aggregate, it turns out that (by
treating aggregates as primitive constructs) the
computation is not harder than for propositional theories.

\begin{proposition}\label{prop:complexaggr}
If, for every aggregate, computing $op(W)\prec N$ requires polynomial time
then the existence of a stable model of a theory with aggregates is a
$\Sigma_2^P$-complete problem.
\end{proposition}

For a nondisjunctive program with nested expressions the
existence of a stable model is NP-complete. If we allow
nonnested aggregates in the body, for instance by allowing rules
$$
A_1\wedge\cdots\wedge A_n\implies a
$$
($A_1,\dots,A_n$ are aggregates and $a$ is an atom or $\bot$)
then the complexity increases to $\Sigma_2^P$. This follows from
Lemma~\ref{lemma:disjunction}, since, in~(\ref{translaggr}),
each formula $l_i$ is the propositional representation of
$sum\langle \{l_i=1\}\rangle\geq 1$; similarly,
each $a_j\implies a_i$ is the propositional representation of
$sum\langle\{a_j=-1,a_i=1\}\rangle\geq 0$.

However, if we allow monotone and antimonotone aggregates only
--- even nested --- in the antecedent, we are in class NP.

\begin{proposition}\label{prop:complexmonaggr}
Consider theories with aggregates consisting of formulas of the form
$$
F\implies a,
$$
where $a$ is an atom or $\bot$, and $F$ contains monotone and
antimonotone aggregates only, no equivalences and no implications other
than negations.
If, for every aggregate, computing $op(W)\prec N$ requires polynomial time
then the problem of the existence of a stable model of theories of this
kind is an NP-complete problem.
\end{proposition}

Similar results have been independently proven in~\cite{cal05} for
FLP-aggregates.

\subsection{Other Formalisms}\label{sec:other}

\begin{figure}
\begin{tabular}{|r|c|c|c|}
\hline
&monotone/antimonotone&generic&anti-chain\\
&aggregates&aggregates&property\\
\hline
weight constraints&NP-complete&NP-complete&NO\\
PDB-aggregates&NP-complete&$\Sigma^2_P$-complete&YES\\
FLP-aggregates&NP-complete&$\Sigma^2_P$-complete&YES\\
our aggregates&NP-complete&$\Sigma^2_P$-complete&NO\\
\hline
\end{tabular}
\caption{Properties of definitions of programs with aggregates, in the case
in which the head of each rule is an atom. We limit the
syntax of our aggregates to the syntax allowed by the other formalisms.
The complexity is relative to the problem of the existence of a stable model.
The anti-chain property holds when no stable model can be a subset of another
one.}
\label{fig:comparison}
\end{figure}
Figure~\ref{fig:comparison} already shows that there are several
differences between the various definitions of an aggregate.
We analyze that more in details in the rest of this section.

\subsubsection{Programs with weight constraints}

Weight constraints are aggregates defined in~\cite{nie00} and
implemented in answer set solver
{\sc smodels}. We simplify the syntax of weight constraints
and of programs with weight constraints for clarity, without reducing
its semantical expressivity.

{\sl Weight constraints} are expressions of the form
\begin{equation}
N\leq \{ l_1 = w_1, \ldots, l_m = w_m\}
\label{wc1}
\end{equation}
and
\begin{equation}
\{ l_1 = w_1, \ldots, l_m = w_m\}\leq N
\label{wc2}
\end{equation}
where
\begin{itemize}
\item $N$ is (a symbol for) a real number,
\item
each of $l_1,\dots,l_n$ is a (symbol for) a literal, and
$w_1,\dots,w_n$ are (symbols for) real numbers.
\end{itemize}
An example of a weight constraint is~(\ref{exweight}).

The intuitive meaning of~(\ref{wc1}) is that 
the sum of the weights $w_i$ for all the $l_i$ that are true is not lower
than $N$. For~(\ref{wc2}) the sum of weights is not greater than $N$. 
Often, $N_1\leq S$ and $S\leq N_2$ are written together as
$N_1\leq S\leq N_2$.
If a weight $w$ is $1$ then the part ``$=w$'' is generally omitted. If
all weights are 1 then a weight constraint is called a {\em cardinality
constraint}.

A {\sl rule with weight constraints} is an
expression of the form
\begin{equation}
\label{wcrule}
a \ar C_1\wedge \cdots\wedge C_n
\end{equation}
where $a$ is an atom or $\bot$, and $C_1,\dots,C_n$ ($n\geq 0$) are
weight constraints.

Finally, a {\sl program with weight constraints} is a set of rules with
weight constraints. 
{\sl Rules/programs with cardinality constraints} are rules/programs
with weight constraints containing cardinality constraints only.

Programs with cardinality/weight constraints can be seen as a generalization
of traditional programs, by identifying each literal $l$ in the body of each
rule with cardinality constraint $1\leq \{l\}$.

The definition of a stable model from~\cite{nie00} requires first the
elimination of negative
weights from weight constraints. This is done by replacing each term
$l_i = w_i$ where $w_i$ is negative with $\o {l_i} = -w_i$
($\o {l_i}$ is the literal complementary to $l_i$) and increasing the
bound by $-w_i$. For instance,
$$
0\leq\{p=2,q=-1\}
$$
is rewritten as
$$
1\leq\{p=2,\neg q=1\}.
$$

Then~\cite{nie00} proposes a definition of a reduct and of a stable model
for programs with weight constraints without negative weights. For this paper,
we prefer showing a translational, equivalent semantics of such programs
from~\cite{fer05b}, that consists in replacing each weight constraint $C$
with a nested expression $[C]$, preserving the stable models of the program:
if $C$ is~(\ref{wc1}) then $[C]$ is
($I\subseteq \{1,\dots, n\}$)
\beq2
\bigvee_{I~:~N\leq \sum_{i \in I} w_i }
\big( \bigwedge_{i \in I} l_i\big)
\eeq2{oldtr1}
and if $C$ is~(\ref{wc2}) then $[C]$ is
\beq2
\neg \bigvee_{I~:~N<\sum_{i \in I} w_i }
\big( \bigwedge_{i \in I} l_i\big).
\eeq2{oldtr2}

It turns out that the way of understanding a weight constraint $C$ of this
paper is not different from $[C]$ when all weights are nonnegative.

\begin{proposition}
In presence of nonnegative weights only,
$[N\leq S]$ is strongly equivalent to 
$sum\langle S\rangle\geq N$,
and $[S\leq N]$ is strongly equivalent to $sum\langle S\rangle\leq N$.
\label{th:weight}
\end{proposition}

From this proposition, Propositions~\ref{th:th1} and~\ref{prop:se} of
this paper, and Theorem~1 from~\cite{fer05b} it follows that our concept of
an aggregate captures the concept of weight constraints defined
in~\cite{nie00} when all weights are nonnegative. It also captures the 
absence of the anti-chain property of its stable models: for instance,
$$
p\ar \{\neg p\}\leq 0
$$
has stable models $\emptyset$ and $\{p\}$ in both formalisms.

When we consider negative weights, however, such correspondence doesn't hold.
For instance,
\beq2
p\ar 0\leq\{p=2,p=-1\},
\eeq2{badex1}
according to~\cite{nie00}, has no stable models, while
\beq2
p\ar sum\langle\{p=2,p=-1\}\rangle\geq 0
\eeq2{badex2}
has stable model $\emptyset$. An explanation of this difference can be seen
in the pre-processing proposed by~\cite{nie00} that eliminates negative
weights. For us, weight constraint $0\leq\{p=2,p=-1\}$, and the result
$1\leq\{p=2,\neg p=1\}$ of eliminating its negative weight, are semantically
different.\footnote{The fact that the process of eliminating negative weights
is somehow unintuitive was already mentioned in~\cite{fer05b} with the same
example proposed in this section.}
Surprisingly, under the semantics of~\cite{nie00}, $0\leq\{p=2,p=-1\}$ is
different from $0\leq\{p=1\}$. In fact,
\beq2
p\ar 0\leq\{p=1\}
\eeq2{badex3}
has stable model $\emptyset$, the same of~(\ref{badex2}), while~(\ref{badex1})
has none. Notice that
summing weights that are all positive or all negative preserves stable models
under both semantics.

The preliminary step of removing negative weights can be seen as a way of
making weight constraints either monotone or antimonotone. This keeps the
problem of the existence of a stable model in class NP, while we have seen
in Section~\ref{sec:aggr-compl} that, under our semantics, even simple
aggregates with the same intuitive meaning of $0\leq \{p=1,q=-1\}$ bring the
same problem to class $\Sigma_2^P$.

\subsubsection{PDB-aggregates}

A {\em PDB-aggregate} is an expression of the form~(\ref{aggregate}), where
$F_1,\dots,F_n$ are literals.
A {\em program with PDB-aggregates} is a set of rules of the form
$$
a\ar A_1\wedge\cdots\wedge A_m,
$$
where $m\geq 0$, $a$ is an atom and
$A_1,\dots,A_m$ are PDB-aggregates.

As in the case of programs with weight constraints, a program with
PDB-aggregates is a generalization of a traditional program, by identifying
each literal $l$ in the bodies of traditional programs by aggregate
$sum\langle \{l=1\}\rangle\geq 1$.

The semantics of~\cite{pel03} for programs with PDB-aggregates
is based on a procedure that transforms programs with such aggregates into
traditional programs.\footnote{A semantics
for such aggregates was proposed in~\cite{den01}, based on
the approximation theory~\cite{den02}. But the first characterization of
PDB-aggregates in terms of stable models is from~\cite{pel03}.~\cite{son07}
independently proposed a similar semantics.}
The procedure can be seen consisting of two parts. The first one
essentially consists in rewriting each aggregate as a nested
expression.\footnote{\cite{pel03} doesn't explicitly mention nested
expressions.} The second part ``unfolds'' each rule into a strongly equivalent
set of traditional rules. For our comparisons, only the first part is needed:
each PDB-aggregate $A$ of the form
$$
op\langle \{l_1=w_1,\dots, l_n=w_n\}\rangle\prec N
$$
is replaced by the following
nested expression $A_{tr}$
$$
\bigvee_{I_1,I_2: I_1\subseteq I_2\subseteq \{1,\dots, n\}
\text{ and for all $I$
such that $I_1\subseteq I\subseteq I_2$,
$op(W_I)\prec N$}} G_{(I_1,I_2)}
$$
where $W_I$ stands for the multiset $\{w_i: i\in I\}$, and
$G_{(I_1,I_2)}$ stands for
$$
\bigwedge_{i\in I_1} l_i,
\bigwedge_{i\in \{1,\dots,n\}\setminus I_2} \o {l_i}.
$$

For instance, for the PDB-aggregate $A=sum\langle \{p=-1,q=1\}\rangle \geq 0$,
if we take $F_1=p$, $F_2=q$ then the pairs $(I_1,I_2)$ that ``contribute''
to the disjunction in $A_{tr}$ are
$$
\begin{array}c
(\emptyset,\emptyset)\qquad (\{2\},\{2\})\qquad (\{1,2\},\{1,2\})\qquad 
(\emptyset,\{2\})\qquad (\{2\},\{1,2\}).
\end{array}
$$
The corresponding nested expressions $G_{(I_1,I_2)}$ are
$$
\begin{array}c
\neg p\wedge \neg q\qquad q\wedge\neg p\qquad p\wedge q\qquad
\neg p\qquad q.
\end{array}
$$
It can be shown, using strong equivalent transformations
(see Proposition~\ref{prop:se}) that
the disjunction of such nested expressions can be rewritten as $\neg p\vee q$.

In case of monotone and antimonotone PDB-aggregates and in the absence of
negation as failure, the semantics of Pelov {\it et al.} is equivalent to
ours.

\begin{proposition}
\label{prop:pelov}
For any monotone or antimonotone PDB-aggregates $A$ of the
form~(\ref{aggregate}) where $F_1,\dots, F_n$ are atoms,
$A_{tr}$ is strongly equivalent to~(\ref{mainaggr}).
\end{proposition}

The claim above is generally not true when either the aggregates are
not monotone or antimonotone, or when some formula in the aggregate is
a negative literal. Relatively to aggregates that are neither monotone nor
antimonotone, the semantics of~\cite{pel03} seems to have the same
unintuitive behaviour of~\cite{nie00}: for instance, according
to~\cite{pel03},~(\ref{badex2}) has no stable models while
$$
p\ar sum\langle\{p=1\}\rangle\geq 0
$$
has stable model $\{p\}$.

To illustrate the problem with negative literals, consider the following $\Pi$:
\beq2
\begin{array}l
p\ar sum\langle\{q=1\}\rangle< 1\\
q\ar \neg p
\end{array}
\eeq2{badpdb3}
and $\Pi'$:
\beq2
\begin{array}l
p\ar sum\langle\{\neg p=1\}\rangle< 1\\
q\ar \neg p
\end{array}
\eeq2{badpdb4}
Intuitively, the two programs should have the same stable models. Indeed,
the operation of replacing $q$ with $\neg p$ in the first rule of $\Pi$
should not affect the stable models since the second rule ``defines'' $q$ as
$\neg p$: it is the only rule with $q$ in the head. However,
under the semantics of~\cite{pel03}, $\Pi$ has stable model $\{p\}$ only and
$\Pi'$ has stable model $\{q\}$ also. Under our semantics, both~(\ref{badpdb3})
and~(\ref{badpdb4}) have stable models $\{p\}$ and $\{q\}$.

Note that already the first rule of~(\ref{badpdb4}) has different stable models
under the two semantics. Under ours, they are $\emptyset$ and $\{p\}$. 
According to~\cite{pel03}, only the empty set is a stable model; it couldn't
have both stable models because stable models as defined in~\cite{pel03}
have the anti-chain property.

\subsubsection{FLP-aggregates}

An {\em FLP-aggregate} is an expression of the form~(\ref{aggregate})
where each of $F_1,\dots,F_n$ is a conjunction of literals.
A {\em program with FLP-aggregates} is a set of rules of the form
\beq2
a_1 \vee\cdots\vee a_n\ar
A_1\wedge\cdots\wedge A_m\wedge \neg A_{m+1}\wedge\cdots\wedge \neg A_p
\eeq2{ruleaggr}
where $n\geq 0, 0\leq m\leq p$, $a_1,\dots,a_n$ are atoms and
$A_1,\dots,A_p$ are FLP-aggregates.

A program with FLP-aggregates
is a generalization of a disjunctive program, by identifying each atom $a$
in the bodies of disjunctive rules by aggregate
$sum\langle \{a=1\}\rangle\geq 1$.

The semantics of~\cite{fab04} defines when a  set of atoms
is a stable model for a program with FLP-aggregates. The definition
of satisfaction of an aggregate is identical to ours. The reduct, however,
is computed differently.
The {\sl reduct} $\Pi^{\un{\un X}}$ of a program $\Pi$ with
FLP-aggregates relative to a  set $X$ of atoms
consists of the rules of the form~(\ref{ruleaggr}) such that $X$ satisfies
its body. Set $X$ is a stable model for $\Pi$ if $X$ is a minimal set
satisfying $\Pi^X$.

For instance, let $\Pi$ be the FLP-program
$$
p\ar sum\langle\{p=2\}\rangle\geq 1.
$$
The only stable model of $\Pi$ is the empty set.
Indeed, since the empty set doesn't satisfy the aggregate,
$\Pi^{\un{\un \emptyset}}=\emptyset$, which has $\emptyset$ as
the unique minimal model;
we can conclude that $\emptyset$ is a stable model of $\Pi$.
On the other hand, $\Pi^{\un{\un {\{p\}}}}=\Pi$ because
$\{p\}$ satisfies the aggregate
in $\Pi$. Since $\emptyset\models \Pi$, $\{p\}$ is not a minimal model of
$\Pi^{\un{\un {\{p\}}}}$ and then it is not a stable model of
$\Pi$.

This definition of a reduct is different from all other definitions of a reduct
described in this paper (and also from many other definitions),
in the sense that it may leave negation
$\neg$ in the body of a rule. For instance, the reduct of $a\ar \neg b$
relative to $\{a\}$ is according to those definitions the fact $a$.
In the theory of FLP-aggregates, the reduct doesn't modify the rule.
On the other hand, this definition of a stable model is equivalent to
the definition of a stable model in the sense of~\cite{gel91b} (and successive
definitions) when applied to disjunctive programs.

Next proposition shows a relationship between our concept of an aggregate and
FLP-aggregates.
An FLP-program is {\em positive} if, in each formula~(\ref{ruleaggr}), $p=m$.

Next proposition shows that our semantics of aggregates is 
essentially an extension of the

\begin{proposition}
\label{aggrsound}
The stable models of a positive FLP-program under our semantics are
identical to its stable models in the sense of~\cite{fab04}.
\end{proposition}

The proposition doesn't apply to arbitrary FLP-aggregates as negation
has different meanings in the two semantics. In case of~\cite{fab04},
$\neg (op\langle S\rangle \prec N)$ is essentially the same as
$op\langle S\rangle \not\prec N$, while we have seen, in
Section~\ref{sec:pt-pf}, that this fact doesn't always hold in our semantics.
The difference in meaning can be seen in the following example. Program
\beq2
\begin{split}
p\ar& \neg q\\
q\ar& sum\langle \{p=1\}\rangle \leq 0\\
\end{split}
\eeq2{flp-problem}
has two stable models $\{p\}$ and $\{q\}$ according to both semantics.
However, if we replace $q$ in the first rule with the body of the second
($q$ is ``defined'' as $sum\langle \{p=1\}\rangle \leq 0$ by the second rule),
we get program
\beq2
\begin{split}
p\ar& \neg (sum\langle \{p=1\}\rangle \leq 0)\\
q\ar& sum\langle \{p=1\}\rangle \leq 0,\\
\end{split}
\eeq2{flp-problem2}
which, according to~\cite{fab04}, has only stable model $\{q\}$. We find it
unintuitive.

It is the first rule of~(\ref{flp-problem2}) that has a different meaning in
the two semantics. The rule alone has different stable models:
according to~\cite{fab04}, its only stable models is $\emptyset$.
Under our semantics, the stable models are $\emptyset$ and $\{p\}$. As they
don't have the anti-chain property, there is no program with FLP-aggregates
that has such stable models under~\cite{fab04}.

As a program with FLP-aggregate can be
easily rewritten as a positive program with FLP-aggregate, our definition
of an aggregate essentially generalizes the one of~\cite{fab04}.

\section{Proofs}\label{sec:proofs}

\subsection{Proofs of Propositions~\ref{lemma:mainlemma} and~\ref{pt-mainth}}

\begin{lemma}\label{conjdisj}
For any formulas $F_1,\dots, F_n$ $(n\geq 0)$, any set $X$ of atoms, and
any connective $\otimes\in\{\vee,\wedge\}$,
$(F_1\otimes\cdots\otimes F_n)^X$ is classically equivalent to
$F_1^X\otimes\cdots\otimes F_n^X$.
\end{lemma}

\begin{proof}
{\bf Case 1}: $X\models F_1\wedge\cdots \wedge F_n$.
Then, by the definition of reduct,
$(F_1\wedge\cdots \wedge F_n)^X=F_1^X\wedge \cdots\wedge F_2^X$.
{\bf Case 2}: $X\not\models F_1\wedge\cdots \wedge F_n$. Then
$(F_1\otimes\cdots\otimes F_n)^X=\bot$; moreover,
one of $F_1,\dots, F_n$ is not satisfied by $X$, so that one of
$F_1^X,\dots, F_n^X$ is $\bot$. The case of disjunction is similar.
\end{proof}

\noindent{\bf Proposition~\ref{lemma:mainlemma}.}
{\it
For any formula $F$ and any HT-interpretation $(X,Y)$, $(X,Y)\models F$ iff
$X\models F^Y$.
}

\begin{proof}
It is sufficient to consider the case when $\Gamma$ is a singleton $\{F\}$,
where $F$ contains only connectives $\wedge$, $\vee$, $\implies$ and $\bot$.
The proof is by structural induction on $F$.
\begin{itemize}
\item
$F$ is $\bot$. $X\not\models \bot$ and $\tuple{X,Y}\not\models \bot$.
\item
$F$ is an atom $a$. $X\models a^Y$ iff $Y\models a$ and $X\models a$. Since
$X\subseteq Y$, this means iff $X\models a$, which is the condition
for which $\tuple{X,Y}\models a$.
\item
$F$ has the form $G\wedge H$.
$X\models (G\wedge H)^Y$ iff $X\models G^Y\wedge H^Y$ by Lemma~\ref{conjdisj},
and then iff $X\models G^Y$ and $X\models H^Y$.
This is equivalent, by induction hypothesis, to say that
$\tuple{X,Y}\models G$ and $\tuple{X,Y}\models H$, and then that
$\tuple{X,Y}\models G\wedge H$.
\item
The proof for disjunction is similar to the proof for conjunction.
\item
$F$ has the form $G\implies H$.
$X\models (G\implies H)^Y$ iff $X\models G^Y\implies H^Y$ and
$Y\models G\implies H$, and then iff
$$
\text{$X\models G^Y$ implies $X\models H^Y$, and $Y\models G\implies H$}.
$$
This is equivalent, by the induction hypothesis, to
$$
\text{$\tuple{X,Y}\models G$ implies $\tuple{X,Y}\models H$,
and $Y\models G\implies H$},
$$
which is the definition of $\tuple{X,Y}\models G\implies H$.
\end{itemize}
\end{proof}

\noindent{\bf Proposition~\ref{pt-mainth}.}
{\it
For any theory, its models in the sense of equilibrium logic are identical
to its stable models.
}

\begin{proof}
A  set $Y$ of atoms is an equilibrium model of $\Gamma$ iff
$$
\text{$\tuple{Y,Y}\models \Gamma$ and, for all proper subsets $X$ of
$Y$, $\tuple{X,Y}\not \models \Gamma$.}
$$
In view of Proposition~\ref{lemma:mainlemma}, this is equivalent to the
condition
$$
\text{$Y\models \Gamma^Y$ and, for all proper subsets $X$ of
$Y$, $X\not \models \Gamma^Y$.}
$$
which means that $Y$ is a stable model of $\Gamma$.
\end{proof}

\subsection{Proof of Propositions~\ref{pt-prop1} and~\ref{th:th1}}

We first need the recursive definition of reduct for programs with
nested expressions from~\cite{lif99d}.
The {\em reduct} $F^{\un X}$ of a nested expression $F$
relative to a  set $X$ of atoms, as follows:
\begin{itemize}
\item
$a^{\un X}=a$, $\bot^{\un X}=\bot$ and
$\top^{\un X}=\top$,
\item
$(F\wedge G)^{\un X}=F^{\un X} \wedge G^{\un X}$ and
$(F\vee G)^{\un X}=F^{\un X} \vee G^{\un X}$,
\item
$(\neg F)^{\un X}=
\begin{cases}
\bot\ , & \text{if $X \models F$}, \cr
\top\ , & \text{otherwise,}\hfill \end{cases}$
\end{itemize}
Then the reduct $(F\ar G)^{\un X}$ of a rule $F\ar G$ with with
nested expression is defined as $F^{\un X}\ar G^{\un X}$, and
the reduct $\Pi^{\un X}$ of a program with nested expressions
as the union of the reduct of its rules.

\begin{lemma}
\label{prop1lemma}
The reduct $F^X$ of a nested expression $F$ is equivalent, in the sense of
classical logic, to the nested expression obtained from $F^{\un X}$
by replacing all atoms that do not belong to $X$ by $\bot$.
\end{lemma}

\begin{proof}
The proof is by structural induction on $F$.
\begin{itemize}
\item
When $F$ is $\bot$ or $\top$ then $F^X=F=F^{\un X}$.
\item
For an atom $a$, $a^{\un X}=a$. The claim is immediate.
\item
Let $F$ be a negation $\neg G$ .
If $X\models G$ then $F^X=\bot=F^{\un X}$; otherwise,
$F^X=\neg \bot=\top=F^{\un X}$.
\item
for $F=G\otimes H (\otimes\in\{\vee,\wedge\})$,
$F^{\un X}$ is $G^{\un X}\otimes H^{\un X}$,
and, by Lemma~\ref{conjdisj}, $F^X$ is equivalent to
$G^X\otimes H^X$. The claim now follows by the induction hypothesis.
\end{itemize}
\end{proof}

\noindent{\bf Proposition~\ref{pt-prop1}.}
{\it
For any program $\Pi$ with nested expressions and any set $X$ of atoms,
$\Pi^X$ is equivalent, in the sense of classical logic,
\begin{itemize}
\item
to $\bot$, if $X\not\models \Pi$, and
\item
to the program obtained from $\Pi^{\un X}$ by replacing all atoms
that do not belong to $X$ by $\bot$, otherwise.
\end{itemize}
}

\begin{proof}
If $X\not\models \Pi$ then clearly $\Pi^X$ contains $\bot$.
Otherwise, $\Pi^X$ consists of formulas $F^X\implies G^X$ for each rule
$G\ar F\in\Pi$, and consequently for each
rule $G^{\un X}\ar F^{\un X}\in\Pi^{\un X}$.
Since each $F$ and $G$ is a nested expression, the claim is immediate
by Lemma~\ref{prop1lemma}.
\end{proof}

\noindent{\bf Proposition~\ref{th:th1}.}
{\it
For any program $\Pi$ with nested expressions, the collection of stable
models of $\Pi$ according to our definition and according to~\cite{lif99d}
are identical.
}

\begin{proof}
If $X\not\models \Pi$ then clearly $\Pi^X$ contains $\bot$, and also
$X\not\models \Pi^{\un X}$ (a well-known property about programs
with nested expressions), so $X$ is not a stable model under either
definitions. Otherwise, by Corollary~\ref{cor1}, the two reducts are
satisfied by the same subsets of $X$. Then $X$ is a minimal set
satisfying $\Pi^X$ iff it is a minimal set satisfying $\Pi^{\un X}$,
and, by the definitions of a stable models $X$ is a stable model of $\Pi$
either for both definitions or for none of them.
\end{proof}

\subsection{Proofs of Propositions~\ref{prop:se}--\ref{prop:constraint}}

\noindent{\bf Proposition~\ref{prop:se}.}
{\it
For any two theories $\Gamma_1$ and $\Gamma_2$, the following conditions
are equivalent:
\begin{enumerate}
\item[(i)]
$\Gamma_1$ is strongly equivalent to $\Gamma_2$,
\item[(ii)]
$\Gamma_1$ is equivalent to $\Gamma_2$ in the logic of here-and-there, and
\item[(iii)]
for each set $X$ of atoms, $\Gamma_1^X$ is equivalent to $\Gamma_2^X$ in
classical logic.
\end{enumerate}
}

\begin{proof}
We will prove the equivalence between~(i) and~(ii) and between~(ii) and~(iii).
We start with the former.
Lemma~4 from~\cite{lif01} tells that, for any two theories, the
following conditions are equivalent:
\begin{enumerate}
\item[(a)]
for every theory $\Gamma$, theories $\Gamma_1\cup\Gamma$ and
$\Gamma_2\cup\Gamma$ have the same equilibrium models, and
\item[(b)]
$\Gamma_1$ is equivalent to $\Gamma_2$ in the logic of here-and-there.
\end{enumerate}
Condition~(b) is identical to~(ii). Condition~(a) can be rewritten,
by Proposition~\ref{pt-mainth}, as
\begin{enumerate}
\item[(a$'$)]
for every theory $\Gamma$, theories $\Gamma_1\cup\Gamma$ and
$\Gamma_2\cup\Gamma$ have the same stable models,
\end{enumerate}
which means that $\Gamma_1$ is strongly equivalent to $\Gamma_2$.

It remains to prove the equivalence between~(ii) and~(iii).
Theory $\Gamma_1$ is equivalent to $\Gamma_2$ in the logic of here-and-there
iff, for every  set $Y$ of atoms, the following condition holds:
$$\text{
for every $X\subseteq Y$,
$\tuple{X,Y}\models \Gamma_1$ iff $\tuple{X,Y}\models \Gamma_2$.
}$$
This condition is equivalent, by Proposition~\ref{lemma:mainlemma}, to
$$\text{
for every $X\subseteq Y$, $X\models \Gamma_1^Y$ iff $X\models \Gamma_2^Y$.
}$$
Since $\Gamma_1^Y$ and $\Gamma_2^Y$ contain atoms from $Y$ only (the
other atoms are replaced by $\bot$ in the reduct), this last condition
expresses equivalence between $\Gamma_1^Y$ and $\Gamma_2^Y$.
\end{proof}

\begin{lemma}\label{lemma:head}
For any theory $\Gamma$, let $S$ be a set of atoms that contains all head
atoms of $\Gamma$.
For any set $X$ of atoms, if $X\models\Gamma$ then $X\cap S\models \Gamma^X$.
\end{lemma}

\begin{proof}
It is clearly sufficient to
prove the claim for $\Gamma$ that is a singleton $\{F\}$.
The proof is by induction on $F$.
\begin{itemize}
\item
If $F=\bot$ then $X\not\models F$, and the claim is trivial.
\item
For an atom $a$, if $X\models a$ then $a^X=a$, but also $a\in S$, so that
$X\cap S\models a^X$.
\item
If $X\models G\wedge H$ then $X\models G$ and $X\models H$. Consequently,
by induction hypothesis, $X\cap S\models G^X$ and
$X\cap S\models H^X$. It remains to notice that $(G\wedge H)^X=G^X\wedge H^X$.
\item
The case of disjunction is similar to the case of conjunction.
\item
If $X\models G\implies H$ then $(G\implies H)^X=G^X\implies H^X$. Assume that
$X\cap S\models G^X$. Consequently $G^X\not = \bot$ and then $X\models G$.
It follows that, since $X\models G\implies H$, $X\models H$.
Since $S$ contains all head atoms of $H$, the claim follows by the
induction hypothesis.
\end{itemize}
\end{proof}

\begin{lemma}\label{lemma:satsame}
For any theory $\Gamma$ and any set $X$ of atoms,
$X\models \Gamma^X$ iff $X\models\Gamma$.
\end{lemma}

\begin{proof} Reduct $\Gamma^X$ is obtained from~$\Gamma$ by replacing some
subformulas that are not satisfied by~$X$ with $\bot$.
\end{proof}

\noindent{\bf Proposition~\ref{prop:head}.}
{\it
Each stable model of a theory $\Gamma$ consists of head atoms of $\Gamma$.
}

\begin{proof}
Consider any theory $\Gamma$, the set $S$ of head atoms of $\Gamma$, and
a stable model $X$ of~$\Gamma$. By Lemma~\ref{lemma:satsame},
$X\models \Gamma$, so that, by Lemma~\ref{lemma:head},
$X\cap S\models \Gamma^X$. Since $X\cap S\subseteq X$ and no proper
subset of $X$ satisfies $\Gamma^X$, it follows that $X\cap S=X$, and
consequently that $X\subseteq S$.
\end{proof}

\noindent{\bf Proposition~\ref{prop:constraint}.}
{\it
For every two propositional theories $\Gamma_1$ and $\Gamma_2$ such that
$\Gamma_2$ has no head atoms, a  set $X$ of atoms
is a stable model of $\Gamma_1\cup \Gamma_2$ iff $X$ is a stable model of
$\Gamma_1$ and $X\models \Gamma_2$.
}

\begin{proof}
If $X\models \Gamma_2$ then $\Gamma_2^X$ is satisfied by every subset of
$X$ by Lemma~\ref{lemma:head}, so that
$(\Gamma_1\cup \Gamma_2)^X$ is classically equivalent to $\Gamma_1^X$;
then clearly $X$ is a stable model of $\Gamma_1\cup \Gamma_2$ iff it
is a stable model of $\Gamma_1$. Otherwise, $\Gamma_2^X$ contains $\bot$,
and $X$ cannot be a stable model of $\Gamma_1\cup \Gamma_2$.
\end{proof}

\subsection{Proofs of Propositions~\ref{prop:expl} and~\ref{prop:split}}

We start with the proof of Proposition~\ref{prop:split}. Some lemmas are
needed.

\begin{lemma}\label{lemma:split2}
If $X$ is a stable model of $\Gamma$ then $\Gamma^X$ is equivalent to
$X$.
\end{lemma}

\begin{proof}
Since all atoms that occur in $\Gamma^X$ belong to~$X$,
it is sufficient to show that the formulas are satisfied by the same
subsets of $X$. By the definition of a stable model, the only subset of~$X$
satisfying~$\Gamma^X$ is~$X$.
\end{proof}

%

\begin{lemma}\label{lemma:split0}
Let~$S$ be a set of atoms that contains all atoms that occur in a
theory~$\Gamma_1$ but does not contain any head atoms of a theory~$\Gamma_2$.
For any set $X$ of atoms, if $X$ is a stable model of
$\Gamma_1\cup \Gamma_2$ then $X\cap S$ is a stable model of $\Gamma_1$.
\end{lemma}

\begin{proof}
Since $X$ is a stable model of $\Gamma_1\cup \Gamma_2$,
$X\models \Gamma_1$, so that $X\cap S\models \Gamma_1$, and, by
Lemma~\ref{lemma:satsame}, $X\cap S\models \Gamma_1^{X\cap S}$.
It remains to show that no proper subset~$Y$ of $X\cap S$ satisfies
$\Gamma_1^{X\cap S}$.  Let $S'$ be the set of head atoms of $\Gamma_2$,
and let $Z$ be $X\cap (S'\cup Y)$.
We will show that $Z$ has the following properties:
\begin{enumerate}
\item[(i)]
$Z\cap S=Y$;
\item[(ii)]
$Z\subset X$;
\item[(iii)]
$Z\models \Gamma_2^X$.
\end{enumerate}
To prove~(i), note that since $S'$ is disjoint from~$S$, and $Y$ is a subset
of $X\cap S$,
$$
Z\cap S=X\cap (S'\cup Y)\cap S=X\cap Y\cap S=(X\cap S)\cap Y=Y.
$$
To prove~(ii), note that set~$Z$ is clearly a subset of~$X$.  It cannot be
equal to~$X$, because otherwise we would have, by~(i),
$$Y=Z\cap S=X\cap S;$$
this is impossible, because~$Y$ is a proper subset of~$X\cap S$.
Property~(iii) follows from Lemma~\ref{lemma:head}, because
$X\models \Gamma_2$, and $S'\cup Y$ contains all head atoms of $\Gamma_2$.

Since~$X$ is a stable model of $\Gamma_1\cup \Gamma_2$, from property~(ii)
we can conclude that $Z\not\models (\Gamma_1\cup \Gamma_2)^X$. Consequently,
by property~(iii), $Z\not\models \Gamma_1^X$.  Since all atoms that occur
in $\Gamma_1$ belong to $S$, $\Gamma_1^X=\Gamma_1^{X\cap S}$, so that
$Z\not\models \Gamma_1^{X\cap S}$.  Since all atoms
that occur in $\Gamma_1^{X\cap S}$ belong to~$S$, it follows
that $Z\cap S\not\models \Gamma_1^{X\cap S}$.  By property~(i), we conclude
that $Y\not\models \Gamma_1^{X\cap S}$.
\end{proof}

%

\noindent
{\bf Proposition~\ref{prop:split}} (Splitting Set Theorem).
{\it
Let $\Gamma_1$ and $\Gamma_2$ be two theories such that no atom
occurring in $\Gamma_1$ is a head atom of $\Gamma_2$. Let $S$ be
a set of atoms containing all head atoms of $\Gamma_1$ but no head atoms
of $\Gamma_2$.
A  set $X$ of atoms is a stable model of $\Gamma_1\cup \Gamma_2$
iff $X\cap S$ is a stable model of $\Gamma_1$ and $X$ is a stable model
of $(X\cap S)\cup\Gamma_2$.
}

\begin{proof}
We first prove the claim in the case when $S$ contains all atoms of $\Gamma_1$.
If $X\cap S$ is not a stable model of $\Gamma_1$ then $X$ is not a stable
model of $\Gamma_1\cup \Gamma_2$ by Lemma~\ref{lemma:split0}. Now suppose that
$X\cap S$ is a stable model of $\Gamma_1$. Then, by Lemma~\ref{lemma:split2},
$\Gamma_1^{X\cap S}$ is equivalent to $X\cap S$.
Consequently,
$$
\begin{array}r
(\Gamma_1\cup \Gamma_2)^X \;=\; \Gamma_1^X\cup \Gamma_2^X \;=
\; \Gamma_1^{X\cap S}\cup \Gamma_2^X
 \;\eq\; (X\cap S)\;\cup\; \Gamma_2^X\qquad\\
 \;=\; (X\cap S)^X\cup \Gamma_2^X
 \;=\; \big((X\cap S)\cup \Gamma_2\big)^X
\end{array}
$$
We can conclude that $X$ is a stable model of $\Gamma_1\cup \Gamma_2$ iff
$X$ is a stable model of $\Gamma_2\cup (X\cap S)$.

The most general case remains. Let $S_1$ be the set of all atoms in $\Gamma_1$
(the value of $S$ for which we have already proved the claim).
In view of the special case described above, it is sufficient to show that,
for any set $S$ of atoms that respects the hypothesis conditions,
\beq2
\text{$X\cap S_1$ is a stable model of $\Gamma_1$ and $X$ is a stable model
of $(X\cap S_1)\cup\Gamma_2$}
\eeq2{splitcond1}
holds iff
\beq2
\text{$X\cap S$ is a stable model of $\Gamma_1$ and $X$ is a stable model
of $(X\cap S)\cup\Gamma_2$}.
\eeq2{splitcond2}
Assume~(\ref{splitcond1}). Sets $S$ and $S_1$ differ only for sets of atoms
that are not head atoms of $\Gamma_1$. Consequently, since $X\cap S_1$ is a
stable model of $\Gamma_1$, it follows from Proposition~\ref{prop:head} that
$X\cap S_1=X\cap S$. We can then conclude that~(\ref{splitcond2}) follows
from~(\ref{splitcond1}). The proof in the opposite direction is similar.
\end{proof}

\begin{lemma}
\label{lemma:replace}
Let $\Gamma$ be a theory, and let $Y$ and $Z$ be two disjoint sets of atoms
such that no atom of $Z$ is an head atoms of $\Gamma$. Let $\Gamma'$ a theory
obtained from $\Gamma$ by replacing occurrences of atoms of $Y$ with $\top$ and
occurrences of atoms of $Z$ with $\bot$. Then
$\Gamma\cup Y$ and $\Gamma'\cup Y$ have the same stable models.
\end{lemma}

\begin{proof}
Atoms of $Z$ are not head atoms of $\Gamma\cup Y$. Consequently,
by Proposition~\ref{prop:head}, every stable model of $\Gamma\cup Y$
is disjoint from $Z$. It follows, by Proposition~\ref{prop:constraint},
that $\Gamma\cup Y$ has the same stable models of
$$
\Gamma \cup Y \cup \{\neg a: a\in Z\}.
$$
Similarly, $\Gamma'\cup Y$ has the same stable models of
$$
\Gamma' \cup Y \cup \{\neg a: a\in Z\}.
$$
It is a known property that the two theories above are equivalent to each
other in intuitionistic logic, and then in the logic-of-here-and-there.
Consequently, by Proposition~\ref{prop:se}, they are strongly equivalent to
each other, and we can conclude that they have the same stable models.
\end{proof}

\noindent{\bf Proposition~\ref{prop:expl}.}
{\it
Let $\Gamma$ be any propositional theory, and $Q$ a set of atoms
not occurring in $\Gamma$. For each $q\in Q$, let $Def(q)$ be a
formula that doesn't contain any atoms from $Q$. Then
$X\mapsto X\setminus Q$ is a 1--1 correspondence between the stable models of
$\Gamma\cup \{Def(q)\implies q: q\in Q\}$ and the stable models of $\Gamma$.
}

\begin{proof}
Let $\Gamma_2$ be $\{Def(q)\implies q: q\in Q\}$. Since $Q$ contains
all head atoms of $\Gamma_2$ but no atom occurring in $\Gamma$ then, by
the splitting set theorem (Proposition~\ref{prop:split}), (``s.m.'' stands
for ``a stable model'')
\beq2
\text{$X$ is s.m. of $\Gamma\cup \Gamma_2$ iff $X\setminus Q$ is s.m. of
$\Gamma$ and $X$ is s.m. of $(X\setminus Q)\cup\Gamma_2$.}
\eeq2{split4explicit}
Clearly, if $X$ is a stable model of $\Gamma\cup \Gamma_2$ then
$X\setminus Q$ is a stable model of $\Gamma$, which proves one of the
two directions of the 1--1 correspondence in the claim.
Now take any stable model $Y$ of
$\Gamma$. We need to show that there is exactly one stable model $X$ of
$\Gamma\cup \Gamma_2$ such that $X\setminus Q=Y$. In view
of~(\ref{split4explicit}), it is sufficient to show that
$$
Z=Y\cup\{q\in Q: Y\models Def(q)\}
$$
is the only stable model $X$ of $Y\cup \Gamma_2$, and that
$Z\setminus Q=Y$. This second condition can be easily verified.
Now consider $Y\cup \Gamma_2$. By Lemma~\ref{lemma:replace},
$Y\cup \Gamma_2$ has the same stable models of
$$
Y\cup \{Def(q)'\implies q: q\in Q\},
$$
where $Def(q)'$ is obtained from $Def(q)$ by replacing all occurrences of
atoms in it with $\top$ if the atom replaced belongs to $Y$, and with $\bot$
otherwise. This theory can be further simplified into theory $Z$.
Indeed, $Def(q)'$ doesn't contain atoms, and then it is strongly
equivalent to $\top$ or $\bot$. In particular, if $Y\models Def(q)$ then
$Def(q)'$ is strongly equivalent to $\top$, and then $Def(q)'\implies q$
is strongly equivalent to $q$. Otherwise, $Def(q)'$ is strongly equivalent
to $\bot$, and then $Def(q)'\implies q$ is strongly equivalent to $\top$.
As $Z$ is a  set of atoms, it is
easy to verify that its only stable model is $Z$ itself.
\end{proof}

\subsection{Proof of Proposition~\ref{prop:compl}}

In order to prove the Completion Lemma, we will need the following lemma.

\begin{lemma}
\label{weaklyposneg}
Take any two  sets $X$, $Y$ of atoms such that $Y\subseteq X$.
For any formula $F$ and any set $S$ of atoms,
\begin{enumerate}
\item[(a)]
if each positive occurrence of an atom from $S$ in $F$ is in the
scope of negation and $Y\models F^X$ then $Y\setminus S \models F^X$, and
\item[(b)]
if each negative occurrence of an atom from $S$ in $F$ is in the
scope of negation and $Y\setminus S \models F^X$ then $Y \models F^X$.
\end{enumerate}
\end{lemma}

\begin{proof}
\begin{itemize}
\item
If $X\not\models F$ then $F^X=\bot$, and the claim is trivial. This covers
the case in which $F=\bot$.
\item
If $X\models F$ and $F$ is an atom $a$ then claim~(b) holds because if
$a\in Y\setminus S$ then
$a\in Y$. For claim~(a), if $a\not\in S$ and $a\in Y$ then
$a\in Y\setminus S$.
\item
If $X\models F$ and $F$ is a conjunction or a disjunction,
the claim is almost immediate
by Lemma~\ref{conjdisj} and induction hypothesis.
\item
The case in which $X\models F$ and $F$ has the form $G\implies H$ remains.
Clearly, $(G\implies H)^X=G^X\implies H^X$.
{\bf Case 1}. If $G\implies H$ is a negation (that is, $H=\bot$)
then, since $X\models F$, $X\not\models G^X$ and then $F^X=\top$, and the
claims clearly follows. {\bf Case 2}: $H\not=\bot$.
We describe a proof of claim~(a). The proof for~(b) is similar.
Assume that no atom from $S$ has positive occurrences in $G\implies H$
outside the scope of the negation, that $Y\models G^X\implies H^X$, and that
$Y\setminus S \models G^X$. We want to prove that $Y\setminus S \models H^X$.
Notice that no atom from $S$ has negative occurrences in $G$ outside
the scope of negation;
consequently, by the induction hypothesis (claim (b)), $Y\models G^X$.
On the other hand, $Y\models (G\implies H)^X$, so that $Y\models H^X$. Since 
no atom from $S$ has positive occurrences in $H$ outside the scope of
negation, we can conclude that $Y\setminus S \models H^X$ by induction
hypothesis (claim (a)).
\end{itemize}
\end{proof}

\noindent{\bf Proposition~\ref{prop:compl}} (Completion Lemma)
{\it
Let $\Gamma$ be any propositional theory, and $Q$ a set of atoms
that have positive occurrences in $\Gamma$ only in the scope of negation.
For each $q\in Q$, let $Def(q)$ be a formula such that all negative
occurrences of atoms from $Q$ in $Def(q)$ are in the scope
of negation. Then $\Gamma\cup \{Def(q)\implies q: q\in Q\}$ and
$\Gamma\cup \{Def(q)\eq q : q\in Q\}$ have the same stable models.
}

\begin{proof}
Let $\Gamma_1$ be $\Gamma\cup \{Def(q)\implies q~:~q\in Q\}$ and
let $\Gamma_2$ be $\Gamma_1\cup \{q\implies Def(q)~:~q\in Q\}$.
We want to prove that a  set $X$ of atoms is
a stable model of both theories or for none of them.
Since $\Gamma_1^X\subseteq \Gamma_2^X$, $\Gamma_2^X$ entails $\Gamma_1^X$.
If the opposite entailment holds also then we clearly have that
$\Gamma_2^X$ and $\Gamma_1^X$ are satisfied by the same subsets of $X$,
and the claim immediately follows. Otherwise, for some $Y\subseteq X$,
$Y\not\models \Gamma_2^X$ and $Y\models\Gamma_1^X$.
First of all, that means that $X\models \Gamma_1$,
so that $\Gamma_1^X$ is equivalent to
$$\Gamma^X\cup \{ Def(q)^X\implies q~:~q\in Q\cap X\}.$$
Secondly, set $Y$ is one of the sets $Y'$ having the following properties:
\begin{enumerate}
\item[(i)]
$Y'\setminus Q=Y\setminus Q$, and
\item[(ii)]
$Y'\models Def(q)^X\implies q$ for all $q\in Q\cap X$.
\end{enumerate}
Let $Z$ be the intersection of such sets $Y'$, and
let $\Delta$ be $\{q\implies Def(q)^X~:~q\in Q\cap X\}$.
Set $Z$ has the following properties:
\begin{enumerate}
\item[(a)]
$Z\subseteq Y$,
\item[(b)]
$Z\models \Gamma_1^X$, and
\item[(c)]
$Z\models \Delta$.
\end{enumerate}
Indeed, claim~(a) holds since $Y$ is one of the elements $Y'$
of the intersection. To prove~(b), first of all, we observe that
$Z\setminus Q=Y\setminus Q$, so that, by~(a), there is a set $S\subseteq Q$
such that $Z=Y\setminus S$;
as $Y\models \Gamma^X$ and $\Gamma$ has all positive occurrences of
atoms from $S\subseteq Q$ in the scope of negation, it follows that
$Z\models \Gamma^X$ by Lemma~\ref{weaklyposneg}(a).
It remains to show that, for any $q$,
if $Z\models Def(q)^X$ then $q\in Z$. Assume that $Z\models Def(q)^X$. Then,
since $Def(q)$ has all negative occurrences of atoms from $Q$ in the
scope of negation, and since all $Y'$ whose intersection generate $Z$ are
superset of $Z$ with $Y'\setminus Z\subseteq Q$, all those $Y'$ satisfy
$Def(q)^X$ by Lemma~\ref{weaklyposneg}. By property~(ii), we have that
$q\in Y'$ for all $Y'$, and then $q\in Z$.

It remains to prove claim~(c).
Take any $q\in Z$ that belongs to $Q\cap X$.
Set $Y'=Z\setminus \{q\}$ satisfies condition~(i),
but it cannot satisfy~(ii), because sets $Y'$ that satisfy~(i) and~(ii)
are supersets of $Z$ by construction of $Z$. Consequently,
$Y'\not\models Def(q)^X$. Since all positive occurrences of atom $q$
in $Def(q)$ are in the scope of negation and $Y'=Z\setminus \{q\}$,
we can conclude that $Z\not\models Def(q)^X$ by Lemma~\ref{weaklyposneg}
again.

Now consider two cases. If $X\not\models \Gamma_2$ then clearly $X$
is not a stable model of $\Gamma_2$. It is not a stable model of $\Gamma_1$
as well. Indeed, since $X\models \Gamma_1$, we have that, for some
$q\in Q\cap X$, $X\not\models Def(q)$. Consequently, $Def(q)^X=\bot$ and
then $X\not\models \Delta$, but, since $Z\models \Delta$ by~(c) and
$Z\subseteq Y\subseteq X$ by~(a), $Z$ is a proper subset of $X$. Since
$Z\models \Gamma_1^X$ by~(b), $X$ is not a stable model of $\Gamma_1$.

In the other case ($X\models \Gamma_2$) it is not hard to see that
$\Gamma_2^X$ is equivalent to $\Gamma_1^X\cup\Delta$. We have that
$Z\models \Gamma_1^X$ by~(b), and then $Z\models \Gamma_2^X$ by~(c).
Since $Y\not\models \Gamma_2^X$, $Z\not = Y$. On the other hand,
$Z\subseteq Y\subseteq X$ by~(a).
This means that $Z$ is a proper subset of $X$ that satisfies
$\Gamma_1^X$ and $\Gamma_2^X$, and we can conclude that $X$ is not an
stable model of any of $\Gamma_1$ and $\Gamma_2$.
\end{proof}

\subsection{Proof of Proposition~\ref{prop:disjunction}}

\noindent{\bf Lemma~\ref{lemma:disjunction}.}
{\it
Rule
\beq2
l_1\wedge\cdots\wedge l_m \implies a_1\vee\cdots\vee a_n
\eeq2{transldisjrule2}
($n>0, m\geq 0$) where $a_1,\dots,a_n$ are atoms and $l_1,\dots,l_m$ are
literals, is strongly equivalent to the set of $n$ implications $(i=1,\dots,n)$
\beq2
(l_1\wedge\cdots\wedge l_m\wedge 
(a_1\implies a_i)\wedge \cdots \wedge (a_n\implies a_i)) \implies a_i.
\eeq2{transldisjrule}
}

\begin{proof}
Let $F$ be~(\ref{transldisjrule2}) and 
$G_i$ $(i=1,\dots,n)$ be~(\ref{transldisjrule}).
We want to prove that $F$
is strongly equivalent to $\{G_1,\dots,G_n\}$
by showing that $F^X$ is classically equivalent to
$\{G_1^X,\dots,G_n^X\}$. Let $H$ be $l_1\wedge\cdots\wedge l_m$.

{\bf Case~1}: $X\not\models H$. Then the antecedents of $F$ and of
all $G_i$ are not satisfied by $X$. It
is then easy to verify that the reducts of $F$ and of all $G_i$ relative to
$X$ are equivalent to $\top$.
{\bf Case 2}: $X\models H$ and $X\not\models F$. Then clearly $F^X=\bot$.
But, for each $i$,  $G_i^X$ is $\bot$: indeed, since $X\not\models F$,
$X\not\models a_i$ for all $i=1,\dots,n$. It follows that the consequent of
each $G_i$ is not satisfied by $X$, but the antecedent is satisfied, because
$X\models H$ and in each
implication $a_j\implies a_i$ in $G_i$, the antecedent is not satisfied.
{\bf Case~3}: $X\models H$ and $X\models F$.
This means that some of $a_1,\dots,a_n$ belong to $X$.
Assume, for instance, that $a_1,\dots,a_p$
($0<p\leq n$) belong to $X$, and $a_{p+1},\dots,a_n$ don't.
Then $F^X$ is equivalent to $H^X\implies (a_1\vee\cdots\vee a_p)$.
Now consider formula $G_i$. If $i>p$ then the consequent $a_i$ is not
satisfied by $X$, but also the antecedent is not: it contains an
implication $a_1\implies a_i$; consequently $G_i^X$ is $\top$.
On the other hand, if $i\leq p$ then the consequent $a_i$ is satisfied
by $X$, as well as each implication $a_j\implies a_i$ in the antecedent
of $G_i$. After a few simplifications, we can rewrite $G_i^X$ as
$$
(H^X \wedge (a_1\implies a_i)\wedge \cdots \wedge (a_p\implies a_i))
\implies a_i.
$$
It is not hard to see that this formula is classically equivalent to
$$
(H^X \implies (a_1\vee\dots,\vee a_p)
$$
which is equivalent to $F^X$, so that the claim easily follows.
\end{proof}

\noindent{\bf Proposition~\ref{prop:disjunction}.}
{\it
The problem of the existence of a stable model of a theory consisting of
formulas of the form $F\implies a$ and $F\implies \bot$ is
$\Sigma_2^P$-hard.
}

\begin{proof}
The problem is in class $\Sigma_2^P$ because, as mentioned in
Section~{sec:prop-compl}, the same problem for the (larger) class
of arbitrary theories is also in $\Sigma_2^P$~\cite{pea01}. Hardness remains
to be proven.

In view of Lemma~\ref{lemma:disjunction}, we can transform a disjunctive
program into a theory consisting of formulas of the form $F\ar a$, with the
same stable models and in polynomial time. Consequently, as the existence of
a stable model of a disjunctive program is $\Sigma_2^P$-hard by~\cite{eit93a},
the same holds for theories as in the statement of this proposition.
\end{proof}






\subsection{Proof of Propositions~\ref{propaggregate} and~\ref{monotone}}

For the proof of these propositions, we define an {\em extended aggregate} to
be either an aggregate of the form~(\ref{aggregate}), or $\bot$. It is
easy to see, that, for each aggregate $A$ of the form~(\ref{aggregate}) and
any set $X$ of atoms, $A^X$ is an extended aggregate. We also define, for
any extended aggregate $A$, $\hat{A}$ as
\begin{itemize}
\item
the formula~(\ref{mainaggr}) if $A$ has the form~(\ref{aggregate}), and
\item
$\bot$, otherwise.
\end{itemize}

\begin{lemma}
For any extended aggregate $A$, $\hat{A}$ is classically equivalent to $A$.
\label{lemma:propaggregate-a}
\end{lemma}

\begin{proof}
The case $A=\bot$ is trivial. The remaining case is when $A$ is an aggregate.
Consider any possible conjunctive term $H_I$ (where $I\subseteq\{1,\dots,n\}$)
of $\hat{A}$:
$$
\big( \bigwedge_{i\in I} F_i \big) \implies
\big( \bigvee_{i\in \o I} F_i \big).
$$
For each set $X$ of atoms there is exactly one set $I$
such that $X\not\models H_I$: the set $I_X$ that consists of the $i$'s such
that $X\models F_i$. Consequently, for every set $X$ of atoms,
\begin{equation*}
\begin{array}{rcl}
X\models \hat{A}
&\text{iff}& \text{$H_{I_X}$ is not a conjunctive term of $\hat{A}$}\\
&\text{iff}& op(\{w_i~:~i\in I_X\})\prec N\\
&\text{iff}& op(\{w_i~:~X\models F_i\})\prec N\\
&\text{iff}& X\models A.
\end{array}
\end{equation*}
\end{proof}

\begin{lemma}
For any aggregate $A$ and any set $X$ of atoms, $\hat{A}^X$ is classically
equivalent to $\hat{A^X}$.
\label{lemma:aggrformreduct}
\end{lemma}

\begin{proof}
{\bf Case 1}: $X\not\models A$. Then $\hat{A^X} = \hat{\bot} = \bot$.
On the other hand, by Lemma~\ref{lemma:propaggregate-a},
$X\not \models \hat{A}$ so that $\hat{A}^X=\bot$ also.
{\bf Case 2}: $X\models A$. Then $A$ is an aggregate, and, by the definition
of a reduct, $\hat{A^X}$ is
\beq2
\bigwedge_{I\subseteq \{1,\dots,n\}~:~op(\{w_i~:~i\in I\})\not\prec N}
 \big( \big( \bigwedge_{i\in I} F_i^X \big) \implies
\big( \bigvee_{i\in \o I} F_i^X \big) \big).
\eeq2{aggrformreduct}
On the other hand, $\hat{A}^X$ is classically equivalent,
by Lemma~\ref{conjdisj}, to
$$
\bigwedge_{I\subseteq \{1,\dots,n\}~:~op(\{w_i~:~i\in I\})\not\prec N}
 \big( \big( \bigwedge_{i\in I} F_i \big) \implies
\big( \bigvee_{i\in \o I} F_i \big) \big)^X.
$$
Notice that, since $X\models \hat{A}$ by Lemma~\ref{lemma:propaggregate-a},
all implications in the formula above are satisfied by $X$.
Consequently, $\hat{A}^X$ is classically equivalent to
$$
\bigwedge_{I\subseteq \{1,\dots,n\}~:~op(\{w_i~:~i\in I\})\not\prec N}
 \big( \big( \bigwedge_{i\in I} F_i \big)^X \implies
\big( \bigvee_{i\in \o I} F_i \big)^X \big),
$$
and then, by Lemma~\ref{conjdisj} again, to~(\ref{aggrformreduct}).
\end{proof}

\noindent{\bf Proposition~\ref{propaggregate}.}
{\it
Let~$A$ be an aggregate of the form~(\ref{aggregate}) and let $G$ be the
corresponding formula~(\ref{mainaggr}). Then
\begin{enumerate}
\item[(a)]
$G$ is classically equivalent to $A$, and
\item[(b)]
for any set $X$ of atoms,
$G^X$ is classically equivalent to $A^X$.
\end{enumerate}
}

\begin{proof}
Part~(a) is immediate from Lemma~\ref{lemma:propaggregate-a}, as $G = \hat{A}$.
For part~(b), we need to show that $\hat{A}^X$  is classical equivalent to
$A^X$. By Lemma~\ref{lemma:aggrformreduct}, $\hat{A}^X$  is classically
equivalent to $\hat{A^X}$. It remains to notice that $\hat{A^X}$ is classically
equivalent to $A^X$ by Lemma~\ref{lemma:propaggregate-a}.
\end{proof}





\begin{lemma}
\label{satmon}
For any aggregate $op\langle  \{F_1=w_1,\dots, F_n=w_n\}\rangle\prec N$,
formula~(\ref{mainaggr}) is classically equivalent to
\beq2
\bigwedge_{I\subseteq \{1,\dots,n\}~:~op(\{w_i~:~i\in I\})\not\prec N}
 \big( \bigvee_{i\in \o I} F_i \big)
\eeq2{aggrmon}
if the aggregate is monotone, and to
$$
\bigwedge_{I\subseteq \{1,\dots,n\}~:~op(\{w_i~:~i\in I\})\not\prec N}
 \big( \neg \bigwedge_{i\in I} F_i \big)
$$
if the aggregate is antimonotone.
\end{lemma}

\begin{proof}
Consider the case of a monotone aggregate first.
Let $G$ be~(\ref{mainaggr}), and $H$ be~(\ref{aggrmon}).
It is easy to verify that $H$ entails $G$. The opposite direction
remains. Assume $G$, and we want to derive every conjunctive term
\beq2
\bigvee_{i\in \o I} F_i
\eeq2{disjterm}
in~$H$. For every conjunctive term $D$ of the form~(\ref{disjterm}) in~$H$,
$op(\{w_i~:~i\in I\})\not\prec N$.
As the aggregate is monotone then, for every subset $I'$ of $I$,
$op(\{w_i~:~i\in I'\})\not\prec N$, so that the implication
$$
\big( \bigwedge_{i\in I'} F_i \big) \implies
\big( \bigvee_{i\in \o {I'}} F_i \big)
$$
is a conjunctive term of $H$ for all $I'\subseteq I$.
Then, since $\o{I'}=\o I\cup (I\setminus I')$,
(``$\Rightarrow$'' denotes entailment, and ``$\Leftrightarrow$''
equivalence)

\begin{equation*}
\begin{split}
H \Rightarrow&
~\bigwedge_{I'\subseteq I} \big(
\big( \bigwedge_{i\in I'} F_i \big) \implies
\big( \bigvee_{i\in \o {I'}} F_i \big)\big)\\\\
\Leftrightarrow&
~\bigwedge_{I'\subseteq I} \big(
\big( \big( \bigwedge_{i\in {I'}} F_i \big) \wedge
\bigwedge_{i\in I'\setminus I} \neg F_i \big)
\implies
\big( \bigvee_{i\in \o {I}} F_i \big)\big)\\\\
\Leftrightarrow&
~\big(
\bigvee_{I'\subseteq I}
\big( \big( \bigwedge_{i\in {I'}} F_i \big) \wedge
\bigwedge_{i\in I'\setminus I} \neg F_i \big)\big)
\implies D.\\\\
\end{split}
\end{equation*}

The antecedent of the implication is a tautology: for each interpretation
$X$, the disjunctive term relative to $I'=\{i\in I: X\models F_i\}$ is
satisfied by $X$. We can conclude that $H$ entails $D$.

The proof for antimonotone aggregates is similar.
\end{proof}

\noindent{\bf Proposition~\ref{monotone}.}
{\it
For any aggregate $op\langle  \{F_1=w_1,\dots, F_n=w_n\}\rangle\prec N$,
formula~(\ref{mainaggr}) is strongly equivalent to
$$
\bigwedge_{I\subseteq \{1,\dots,n\}~:~op(\{w_i~:~i\in I\})\not\prec N}
 \big( \bigvee_{i\in \o I} F_i \big)
$$
if the aggregate is monotone, and to
$$
\bigwedge_{I\subseteq \{1,\dots,n\}~:~op(\{w_i~:~i\in I\})\not\prec N}
 \big( \neg \bigwedge_{i\in I} F_i \big)
$$
if the aggregate is antimonotone.
}
\begin{proof}
Consider the case of a monotone aggregate first.
Let $G$ be~(\ref{mainaggr}), and $H$ be~(\ref{aggrmon}).
In view of Proposition~\ref{prop:se}, it is sufficient to show
that $G^X$ is equivalent to $H^X$ in classical logic for all sets $X$.
If $X\not\models H$ then also $X\not\models G$ by Lemma~\ref{satmon},
so that both reducts are $\bot$. Otherwise ($X\models H$), by the same
lemma, $X\models G$. Then, by Lemma~\ref{lemma:aggrformreduct}, $G^X$
is classically equivalent to~(\ref{aggrformreduct}).
On the other hand, it is easy to
verify, by applying Lemma~\ref{conjdisj} to $H^X$ twice, that
$H^X$ is classically equivalent to
$$
\bigwedge_{I\subseteq \{1,\dots,n\}~:~op(\{w_i~:~i\in I\})\not\prec N}
 \big( \bigvee_{i\in \o I} F_i^X \big).
$$
The claim now follows from Lemma~\ref{satmon}.

The reasoning for nonmonotone aggregates is similar.
\end{proof}

\subsection{Proof of Proposition~\ref{exprop2}}

Let $\Gamma$ be the theory consisting of formulas~(\ref{joe1})--(\ref{joe4}).

\begin{lemma}
\label{exprop2l1}
For any stable model $X$ of $\Gamma$, $X$ contains an atom $s_i$ iff
$X$ contains an atom $b_j$ such that bid $j$ involves selling object $i$.
\end{lemma}

\begin{proof}
Consider $\Gamma$ as a propositional theory. We notice that
\begin{itemize}
\item
formulas~(\ref{joe3}) can be strongly equivalently grouped
as $m$ formulas ($i=1,\dots,m)$
$$
\big( \bigwedge_{j=1,\dots,n: \text{ object $i$ is part of bid $j$}} b_j\big)
\implies s_i,
$$
and
\item
no other formula of $\Gamma$ contains atoms of the form $s_i$ outside
the scope of negation.
\end{itemize}
Consequently, by the Completion Lemma (Proposition~\ref{prop:compl}),
formulas~(\ref{joe3}) in $\Gamma$ can be replaced
by $m$ formulas ($i=1,\dots,m)$
\beq2
\big( \bigwedge_{j=1,\dots,n: \text{ object $i$ is part of bid $j$}}
b_j\big)\eq s_i.
\eeq2{joe3c}
preserving the stable models. It follows that every stable model of $\Gamma$
must satisfy formulas~(\ref{joe3c}), and the claim immediately follows.
\end{proof}

\noindent{\bf Proposition~\ref{exprop2}.}
{\it
$X\mapsto \{i: b_i\in X\}$ is a 1--1 correspondence between the
stable models of the theory consisting of formulas~(\ref{joe1})--(\ref{joe4})
and a solution to Joe's problem.
}

\begin{proof}
Take any stable model $X$ of $\Gamma$. Since $X$ satisfies
rules~(\ref{joe2}) of $\Gamma$, condition~(a) is satisfied.
Condition~(b) is satisfies as well, because $X$ contains exactly
all atoms $s_i$ sold in some bids by Lemma~\ref{exprop2l1},
and since $X$ satisfies aggregate~(\ref{joe4}) that belongs to $\Gamma$.

Now consider a solution of Joe's problem. This determines which
atoms of the form $b_i$ belongs to a possible
corresponding stable model $X$. Consequently, Lemma~\ref{exprop2l1}
determines also which atoms of the form $s_j$ belong to $X$, reducing
the candidate stable models $X$ to one.
We need to show that this $X$ is indeed a stable model of $\Gamma$.
The reduct $\Gamma^X$ consists of (after a few simplifications)
\begin{enumerate}
\item[(i)]
all atoms $b_i$ that belong to $X$ (from~(\ref{joe1})),
\item[(ii)]
$\top$ from~(\ref{joe2}) since~(a) holds,
\item[(iii)]
(by Lemma~\ref{exprop2l1})
implications~(\ref{joe3}) such that both $b_j$ and $s_i$ belong to $X$,
and
\item[(iv)]
the reduct of~(\ref{joe4}) relative to $X$.
\end{enumerate}
Notice that~(i)--(iii) together are equivalent to $X$, so that every
every proper subset of $X$ doesn't satisfy $\Gamma^X$. It remains to show
that $X\models \Gamma^X$. Clearly, $X$ satisfies (i)--(iii). To show that
$X$ satisfies~(iv) it is sufficient, by Lemma~\ref{lemma:satsame}
(consider~(\ref{joe4}) as a propositional formula), to show
that $X$ satisfies~(\ref{joe4}): it does that by hypothesis~(b).
\end{proof}

\subsection{Proof of Propositions~\ref{prop:complexaggr} and~\ref{prop:complexmonaggr}}

\begin{lemma}\label{lemma:complexaggr}
If, for every aggregate, computing $op(W)\prec N$ requires polynomial time
then
\begin{enumerate}
\item[(a)]
checking satisfaction of a theory with aggregates requires polynomial time,
and
\item[(b)]
computing the reduct of a theory with aggregates requires polynomial
time.
\end{enumerate}
\end{lemma}

\begin{proof}
Part~(a) is easy to verify by structural induction. Computing the
reduct essentially consists of checking satisfaction of subexpressions of
each formula of the theory. Each check doesn't require too much time by~(a).
It remains to notice that each formula with aggregates has a linear number
of subformulas.
\end{proof}

\noindent{\bf Proposition~\ref{prop:complexaggr}.}
{\it
If, for every aggregate, computing $op(W)\prec N$ requires polynomial time
then the existence of a stable model of a theory with aggregates is a
$\Sigma_2^P$-complete problem.
}

\begin{proof}
Hardness follows from the fact that theories with aggregates
are a generalization of theories without aggregates.
To prove inclusion, consider that the existence of a stable model of
a theory $\Gamma$ is equivalent to satisfiability of:
$$
\text{
{\bf exists} $X$ such that {\bf for all} $Y$,
if $Y\subseteq X$ then $Y\models \Gamma^X$ iff $X=Y$
}
$$
It remains to notice that, in view of Lemma~\ref{lemma:complexaggr}, checking
(for any $X$ and $Y$)
$$
\text{
if $Y\subseteq X$ then $Y\models \Gamma^X$ iff $X=Y$
}
$$
requires polynomial time.
\end{proof}

\begin{lemma}\label{lemma:complexmonaggr1}
Let $F$ be a formula with aggregates containing monotone and
antimonotone aggregates only, no equivalences and no implications other
than negations. For any  sets $X$, $Y$ and $Z$ such that $Y\subseteq Z$,
if $Y\models F^X$ then $Z\models F^X$.
\end{lemma}

\begin{proof}
Let $G$ be $F$ with each monotone aggregate replaced by~(\ref{monaggr})
and each antimonotone aggregate replaced by~(\ref{antimonaggr}).
It is easy to verify that $G$ is a nested expression. Nested expressions
have all negative occurrences of atoms in the scope of negation,
so if $Y\models G^X$ then $Z\models G^X$ by Lemma~(\ref{weaklyposneg}).
It remains to notice that $F^X$ and $G^X$ are satisfied by the same sets
of atoms by Propositions~\ref{monotone} and~\ref{propaggregate}.
\end{proof}

\noindent{\bf Proposition~\ref{prop:complexmonaggr}.}
{\it
Consider theories with aggregates consisting of formulas of the form
\beq2
F\implies a,
\eeq2{complexmonaggr}
where $a$ is an atom or $\bot$, and $F$ contains monotone and
antimonotone aggregates only, no equivalences and no implications other than
negations.
If, for every aggregate, computing $op(W)\prec N$ requires polynomial time
then the problem of the existence of a stable model of theories of this
kind is an NP-complete problem.
}

\begin{figure}
\begin{tabular}l
{\bf function} {\tt verifyAS}$(\Gamma,X\}$\\
$\qquad${\bf if $X\not\models \Gamma$ then return false}\\
$\qquad \Delta:=\{F^X\implies a~:~F\implies a\in \Gamma\text{ and } X\models a\}$\\
$\qquad Y:=\emptyset$\\
$\qquad${\bf while} there is a formula $G\implies a\in \Delta$ such that
   $Y\models G$ and $a\not\in Y$\\
$\qquad \qquad Y:=Y\cup\{a\}$\\
$\qquad${\bf end while}\\
$\qquad${\bf if $Y=X$ then return true}\\
$\qquad${\bf return false}\\
\end{tabular}
\caption{A polynomial-time algorithm that checks stable models of special
kinds of theories}
\label{fig:findmodel}
\end{figure}

\begin{proof}
NP-hardness follows from the fact that theories with aggregates are
a generalization of traditional programs, for which the same problem
is NP-complete. For inclusion in NP,
it is sufficient to show that the time required to check if a 
set $X$ of atoms is a stable model of $\Gamma$ is polynomial.
An algorithm that does this test is in Figure~\ref{fig:findmodel}.
It is easy to verify that it is a polynomial time algorithm.
It remains to prove that it is correct. If $X\not\models \Gamma$ then
it is trivial. Now assume that $X\models \Gamma$.
It is sufficient to show that
\begin{enumerate}
\item[(a)]
$\Delta$ is classically equivalent to $\Gamma^X$, and
\item[(b)]
the last value of $Y$ (we call it $Z$) is the unique minimal model of
$\Delta$.
\end{enumerate}
Indeed, for part~(a), we notice that, since
$X\models \Gamma$, $\Gamma^X$ is
$$
\{F^X\implies a^X~:~F\implies a\in \Gamma\text{ and } X\models a\}\cup
\{F^X\implies a^X~:~F\implies a\in \Gamma\text{ and } X\not\models a\}.
$$
The first set is $\Delta$. The second set (which includes the case
in which $a=\bot$) is a set of $\bot\implies\bot$.
Indeed, each $a^X=\bot$, and since $X\models \Gamma$, $X$ doesn't satisfy
any $F$ and then $F^X=\bot$.

For part~(b) it is easy to verify that the {\bf while}
loop iterates as long as $Y\not\models\Delta$, so that $Z\models\Delta$.
Now assume, in sake of contradiction, that there is a set $Z'$
that satisfies $\Delta$ and that is not a superset of $Z$.
Consider, in the execution of the algorithm, the first atom $a\not\in Z'$
added to $Y$, and that value of $Y\subseteq Z'$ to which $a$ has been added
to. This means that $\Delta$ contains a formula $G\implies a$ such
that $Y\models G$.
Recall that $G$ stands for a formula of the form $F^X$, where $F$
is a formula with aggregates with monotone and antimonotone aggregates
only and without implications (other than negations) or equivalences.
Consequently, by Lemma~\ref{lemma:complexmonaggr1}, $Z'\models G$. On the
other hand, $a\not\in Z'$, so $Z'\not\models G\implies a$, contradicting
the hypothesis that $Z'$ is a model of $\Delta$.
\end{proof}

\subsection{Proof of Proposition~\ref{th:weight}}

\begin{lemma}
Let $F$ and $G$ two propositional formulas, and let $F'$ and $G'$ the result
of replacing each occurrence of an atom $a$ in $F$ and $G$ with a propositional
formula $H$. If $F$ and $G$ are strongly equivalent to
each other then $F'$ and $G'$ are strongly equivalent to each other.
\label{nestedschema1}
\end{lemma}

\begin{proof}
It follows from Proposition~\ref{prop:se}, in view of the following fact:
if $F$ and $G$ are equivalent in the logic of here-and-there to
each other then $F'$ and $G'$ are equivalent 
in the logic of here-and-there to each other.
\end{proof}

\begin{lemma}
\label{nestedschema2}
Let $F$ and $G$ be two propositional formulas that are
AND-OR combinations of $\top$, $\bot$ and atoms only.
If $F$ and $G$ are classically equivalent to each other then
they are strongly equivalent to each other also.
\end{lemma}

\begin{proof}
In view of Proposition~\ref{prop:se}, it is sufficient to show
that, for every set $X$ of atoms, $F^X$ is classically equivalent
to $G^X$. By Lemma~\ref{conjdisj} we can distribute the reduct operator
in $F^X$ to its atoms. If follows that $F^X$ is classically
equivalent to $F$ with all occurrences of atoms that don't belong to $X$
replaced by $\bot$, and similarly for $G^X$. The fact that
$F^X$ is classically equivalent to $G^X$ now follows from
the classical equivalence between $F$ and $G$.
\end{proof}

Next Lemma immediately follows from our definition of satisfaction of
aggregates (Section~\ref{sec:defaggr} of this paper), and
the definition of $[L\leq S]$ and $[S\leq U]$ and Proposition~1
from~\cite{fer05b}.

\begin{lemma}
For every weight constraints $L\leq S$ and $S\leq U$ and any 
set $X$ of atoms,
\begin{enumerate}
\item[(a)]
$X\models [L\leq S]$ iff
$X\models sum\langle S\rangle\geq L$, and
\item[(b)]
$X\models [S\leq U]$ iff
$X\models sum\langle S\rangle\leq U$.
\end{enumerate}
\label{lemma:classiceq}
\end{lemma}

\noindent{\bf Proposition~\ref{th:weight}.}
{\it
In presence of nonnegative weights only,
$[N\leq S]$ is strongly equivalent to 
$sum\langle S\rangle\geq N$,
and $[S\leq N]$ is strongly equivalent to $sum\langle S\rangle\leq N$.
}

\begin{proof}
We start with~(a), with the special case when rule elements $F_1,\dots,F_n$
of $S$ are distinct atoms. Since the aggregate is monotone then, by
Lemma~\ref{monotone}, we just need to show that $[N\leq S]$ is strongly
equivalent to~(\ref{monaggr}). As classical equivalence holds between
$[N\leq S]$ and $sum\langle S\rangle\geq N$ by Lemma~\ref{lemma:classiceq},
the same relationship holds
between $[N\leq S]$ and~(\ref{monaggr}). As both formulas are AND-OR
combinations of atoms, the claim follows by Lemma~\ref{nestedschema2}.
The most general case of~(a) follows from the special case, by
Lemma~\ref{nestedschema1}.

For part~(b), we know, by Lemma~\ref{monotone}, that antimonotone
aggregate
$sum\langle S\rangle\leq U$ (written as a formula~(\ref{aggregate}))
is strongly equivalent to formula
$$
\bigwedge_{I\subseteq \{1,\dots,n\}~:~\sum_{i\in I} w_i> U}
\big( \neg \bigwedge_{i\in I} F_i\big).
$$
By applying DeMorgan's law to this last formula (which preserves equivalence
in the logic of here-and-there and then it is a strongly equivalent
transformation by Proposition~\ref{prop:se}) we get $S\leq U$.
\end{proof}

\subsection{Proof of Proposition~\ref{prop:pelov}}

Given a PDB-aggregate of the form~(\ref{aggregate}) and a 
set $X$ of literals, by $I_X$ we denote the set
$\{i\in\{1,\dots,n\}~:~X\models F_i\}$.

\begin{lemma}
\label{lemma-pelov1}
For each PDB-aggregate of the form~(\ref{aggregate}),
a set $X$ of atoms satisfies a formula of the form $G_{(I_1,I_2)}$
iff $I_1\subseteq I_X\subseteq I_2$.
\end{lemma}

\begin{proof}
$$
\begin{array}{rcl}
X\models G_{(I_1,I_2)}
&\text{iff}&
X\models F_i\text{ for all }i\in I_1,\text{ and }
X\not\models F_i\text{ for all }i\in \{1,\dots,n\}\setminus I_2\\
&\text{iff}&
X\models F_i\text{ for all }i\in I_1,\text{ and }
\text{for every $i$ such that $X\models F_i$, $i\in I_2$}\\
&\text{iff}&
I_1\subseteq I_X\text{ and }
I_X\subseteq I_2.
\end{array}
$$
\end{proof}

\begin{lemma}
\label{lemma-pelov2}
For every PDB-aggregate $A$, $A_{tr}$ is classically equivalent
to~(\ref{mainaggr}).
\end{lemma}

\begin{proof}
Consider a  set $X$ of atoms.
By Lemma~\ref{lemma-pelov1}, $X\models A_{tr}$ iff
$$
\text{$X$ satisfies one of the disjunctive terms $G_{(I_1,I_2)}$ of $A_{tr}$}
$$
and then iff
$$
\text{$A_{tr}$ contains a disjunctive term $G_{(I_1,I_2)}$ such
that $I_1\subseteq I_X\subseteq I_2$}.
$$
It is easy to verify that if this condition holds then one of such terms
$G_{(I_1,I_2)}$ is $G_{(I_X,I_X)}$. Consequently,
$$
\begin{array}{rcl}
X\models A_{tr}
&\text{iff}&
\text{$A_{tr}$ contains disjunctive term $G_{(I_X,I_X)}$}\\
&\text{iff}&
op(W_{I_X})\prec N.
\end{array}
$$
We have essentially found that $X\models A_{tr}$ iff $X\models A$. The
claim now follows by Proposition~\ref{propaggregate}(a).
\end{proof}

\begin{lemma}
\label{lemma-pelov3}
For any PDB-aggregate $A$, $A_{tr}$ is strongly equivalent to
\begin{itemize}
\item[(a)]
$$
\bigvee_{I\in \{1,\dots,n\}: op(W_I)\prec N} G_{(I,\{1,\dots,n\})}
$$
if $A$ is monotone, and to
\item[(b)]
$$
\bigvee_{I\in \{1,\dots,n\}: op(W_I)\prec N} G_{(\emptyset,I)}
$$
if it is antimonotone.
\end{itemize}
\end{lemma}

\begin{proof}
To prove~(a), assume that $A$ is monotone. Then, if $A_{tr}$ contains a
disjunctive term $G_{(I_1,I_2)}$ then it contains the disjunctive term
$G_{(I_1,\{1,\dots,n\})}$ as well. Consider also that formula
$G_{(I_1,\{1,\dots,n\})}$ entails $G_{(I_1,I_2)}$ in the logic of
here-and-there.
Then, by Proposition~\ref{prop:se}, we can drop all disjunctive terms of
the form $G_{(I_1,I_2)}$ with $I_2\not=\{1,\dots,n\}$, preserving strong
equivalence. Formula $A_{tr}$ becomes
$$
\bigvee_{I_1\subseteq \{1,\dots, n\}:
\text{ for all $I$
such that $I_1\subseteq I\subseteq \{1,\dots,n\}$,
$op(W_I)\prec N$}} G_{(I_1,\{1,\dots,n\})}.
$$
It remains to notice that, since $A$ is monotone, if
$op(W_{I_1})\prec N$ then $op(W_{I})\prec N$ for all $I$ superset of
$I_1$.

The proof for~(b) is similar.
\end{proof}

\noindent{\bf Proposition~\ref{prop:pelov}}
{\it
For any monotone or antimonotone PDB-aggregates $A$ of the
form~(\ref{aggregate}) where $F_1,\dots, F_n$ are atoms,
$A_{tr}$ is strongly equivalent to~(\ref{mainaggr}).
}

\begin{proof}
Let $S$ be $\{F_1=w_1,\dots, F_n=w_n\}$.
Lemma~\ref{lemma-pelov2} says that $A_{tr}$ is classically
equivalent to~(\ref{mainaggr}) for every formulas $F_1,\dots,F_n$ in $S$. 
We can then prove the claim of this proposition using Lemma~\ref{nestedschema2},
by showing that both $A_{tr}$ and~(\ref{mainaggr})
can be strongly equivalently rewritten as AND-OR combinations of
\begin{itemize}
\item $F_1,\dots,F_n,\top,\bot$, if $A$ is monotone, and
\item $\neg F_1,\dots,\neg F_n,\top,\bot$, if $A$ is antimonotone.
\end{itemize}
About~(\ref{mainaggr}), this has already been shown in the proof of
Proposition~\ref{th:weight}, while, about $A_{tr}$, this is shown by
Lemma~\ref{lemma-pelov3}. Indeed,
each $G_{(I,\{1,\dots,n\})}$ is a (possibly empty) conjunction
of terms of the form $F_i$, and each $G_{(\emptyset,I)}$
is a (possibly empty) conjunction of terms of the form $\neg F_i$, since
each $F_i$ is an atom.
\end{proof}

\subsection{Proof of Proposition~\ref{aggrsound}}

We observe, first of all, that
the definition of satisfaction of FLP-aggregates and FLP-programs
in~\cite{fab04} is equivalent to ours.
The definition of a reduct is different, however.
Next lemma is easily provable by structural induction.

\begin{lemma}
\label{noimplic}
For any nested expression $F$ without negations
and any two sets $X$ and $Y$ of atoms such that $Y\subseteq X$,
$
Y\models F^X\text{ iff } Y\models F.
$
\end{lemma}

\begin{lemma}
\label{simpleaggr}
For any FLP-aggregate $A$ and any set $X$ of atoms,
if $X\models A$ then
$$
Y\models A^X\text{ iff }
Y\models A.
$$
\end{lemma}

\begin{proof}
Let $A$ have the form~(\ref{aggregate}). Since $X\models A$,
$A^X$ has the form
$$
op\langle \{F_1^X=w_1,\dots, F_n^X=w_n\}\rangle\prec N.
$$
In case of FLP-aggregates, each $F_i$ is a conjunction of atoms. Then,
by Lemma~\ref{noimplic}, $Y\models F_i^X$ iff $Y\models F_i$. The claim
immediately follows from the definition of satisfaction of aggregates.
\end{proof}

\noindent{\bf Proposition~\ref{aggrsound}.}
{\it
The stable models of a positive FLP-program under our semantics are
identical to its stable models in the sense of~\cite{fab04}.
}

\begin{proof}
It is easy to see that if $X\not \models \Pi$ then $X\not \models \Pi^X$
and $X\not\models \Pi^{\un{\un X}}$, so that $X$ is not a
stable model under either semantics. Now assume that $X\models \Pi$.
We will show that the two reducts are satisfied by the same subsets of $X$.
It is sufficient to consider the case in which $\Pi$ contains
only one rule
\beq2
A_1\wedge\dots\wedge A_m\implies a_1\vee\dots\vee a_n.
\eeq2{flpprogpos}
If $X\not\models A_1\wedge\dots\wedge A_m$
then $\Pi^{\un{\un X}}=\emptyset$, and $\Pi^X$ is the
tautology
$$
\bot\implies (a_1\vee\dots\vee a_n)^X.
$$
Otherwise, $\Pi^{\un{\un X}}$
is rule~(\ref{flpprogpos}), and $\Pi^X$ is
$$
A_1^X\wedge\dots\wedge A_m^X\implies (a_1\vee\dots\vee a_n)^X.
$$
These two reducts are satisfied by the same subsets of $X$ by 
Lemmas~\ref{noimplic} and~\ref{simpleaggr}.
\end{proof}

\section{Conclusions}\label{sec:conclusions}

We have proposed a new definition of stable model --- for proposition theories
--- that is simple, very general, and that inherits several properties from
logic programs with nested expressions. On top of that, we have defined the
concept of an aggregate, both as an atomic operator and as a propositional
formula. We hope that this very general framework may be useful in the
heterogeneous world of aggregates in answer set programming.

\section*{Acknowledgements}

We thank Vladimir Lifschitz for many useful comments on a draft of this paper.

\bibliography{/u/vl/tex/bib}

\begin{thebibliography}{}

\bibitem[\protect\citeauthoryear{Baral and Uyan}{2001}]{bar01}
Chitta Baral and Cenk Uyan.
\newblock Declarative specification and solution of combinatorial auctions
  using logic programming.
\newblock {\em Lecture Notes in Computer Science}, 2173:186--199, 2001.

\bibitem[\protect\citeauthoryear{Cabalar and Ferraris}{2007}]{cab04}
Pedro Cabalar and Paolo Ferraris.
\newblock Propositional theories are strongly equivalent to logic programs.
\newblock {\em Theory Pract. Log. Program.}, 7(6):745--759, 2007.

\bibitem[\protect\citeauthoryear{Calimeri \bgroup \em et al.\egroup
  }{2005}]{cal05}
Francesco Calimeri, Wolfgang Faber, Nicola Leone, and Simona Perri.
\newblock Declarative and computational properties of logic programs with
  aggregates.
\newblock In {\em Proceedings of International Joint Conference on Artificial
  Intelligence (IJCAI)}, 2005.

\bibitem[\protect\citeauthoryear{Denecker \bgroup \em et al.\egroup
  }{2001}]{den01}
Marc Denecker, Nikolay Pelov, and Maurice Bruynooghe.
\newblock Ultimate well-founded and stable semantics for logic programs with
  aggregates.
\newblock In {\em Proc. ICLP}, pages 212--226, 2001.

\bibitem[\protect\citeauthoryear{Denecker \bgroup \em et al.\egroup
  }{2002}]{den02}
Marc Denecker, V.~Wiktor Marek, and Miros{\l}aw Truszczy\'nski.
\newblock Ultimate approximations in nonmonotonic knowledge representation
  systems.
\newblock In {\em Proc. KR}, pages 177--190, 2002.

\bibitem[\protect\citeauthoryear{Dimopoulos \bgroup \em et al.\egroup
  }{1997}]{dim97}
Yannis Dimopoulos, Bernhard Nebel, and Jana Koehler.
\newblock Encoding planning problems in non-monotonic logic programs.
\newblock In Sam Steel and Rachid Alami, editors, {\em Proceedings of European
  Conference on Planning}, pages 169--181. Springer-Verlag, 1997.

\bibitem[\protect\citeauthoryear{Eiter and Gottlob}{1993}]{eit93a}
Thomas Eiter and Georg Gottlob.
\newblock Complexity results for disjunctive logic programming and application
  to nonmonotonic logics.
\newblock In Dale Miller, editor, {\em Proceedings of International Logic
  Programming Symposium (ILPS)}, pages 266--278, 1993.

\bibitem[\protect\citeauthoryear{Erdem \bgroup \em et al.\egroup
  }{2000}]{erd00}
Esra Erdem, Vladimir Lifschitz, and Martin Wong.
\newblock Wire routing and satisfiability planning.
\newblock In {\em Proceedings of International Conference on Computational
  Logic}, pages 822--836, 2000.

\bibitem[\protect\citeauthoryear{Erdem \bgroup \em et al.\egroup
  }{2003}]{erd03a}
Esra Erdem, Vladimir Lifschitz, Luay Nakhleh, and Donald Ringe.
\newblock Reconstructing the evolutionary history of {I}ndo-{E}uropean
  languages using answer set programming.
\newblock In {\em Proceedings of International Symposium on Practical Aspects
  of Declarative Languages (PADL)}, pages 160--176, 2003.

\bibitem[\protect\citeauthoryear{Erdo\u{g}an and Lifschitz}{2004}]{erdo04}
Selim~T. Erdo\u{g}an and Vladimir Lifschitz.
\newblock Definitions in answer set programming.
\newblock In Vladimir Lifschitz and Ilkka Niemel{\"a}, editors, {\em
  Proceedings of International Conference on Logic Programming and Nonmonotonic
  Reasoning (LPNMR)}, pages 114--126, 2004.

\bibitem[\protect\citeauthoryear{Faber \bgroup \em et al.\egroup
  }{2004}]{fab04}
Wolfgang Faber, Nicola Leone, and Gerard Pfeifer.
\newblock Recursive aggregates in disjunctive logic programs: Semantics and
  complexity.
\newblock In {\em Proceedings of European Conference on Logics in Artificial
  Intelligence (JELIA)}, 2004.
\newblock Revised version: {\tt
  http://www.wfaber.com/research/papers/jelia2004.pdf}.

\bibitem[\protect\citeauthoryear{Ferraris and Lifschitz}{2005}a]{fer05e}
Paolo Ferraris and Vladimir Lifschitz.
\newblock Mathematical foundations of answer set programming.
\newblock In {\em We Will Show Them! Essays in Honour of Dov Gabbay}, pages
  615--664. King's College Publications, 2005.

\bibitem[\protect\citeauthoryear{Ferraris and Lifschitz}{2005}b]{fer05b}
Paolo Ferraris and Vladimir Lifschitz.
\newblock Weight constraints as nested expressions.
\newblock {\em Theory and Practice of Logic Programming}, 5:45--74, 2005.

\bibitem[\protect\citeauthoryear{Ferraris}{2005}]{fer05}
Paolo Ferraris.
\newblock Answer sets for propositional theories.
\newblock In {\em Proceedings of International Conference on Logic Programming
  and Nonmonotonic Reasoning (LPNMR)}, pages 119--131, 2005.

\bibitem[\protect\citeauthoryear{Ferraris}{2007}]{fer07c}
Paolo Ferraris.
\newblock {\em Expressiveness of answer set languages}.
\newblock PhD thesis, University of Texas at Austin, 2007.
\newblock PhD thesis.

\bibitem[\protect\citeauthoryear{Gelfond and Lifschitz}{1988}]{gel88}
Michael Gelfond and Vladimir Lifschitz.
\newblock The stable model semantics for logic programming.
\newblock In Robert Kowalski and Kenneth Bowen, editors, {\em Proceedings of
  International Logic Programming Conference and Symposium}, pages 1070--1080.
  MIT Press, 1988.

\bibitem[\protect\citeauthoryear{Gelfond and Lifschitz}{1991}]{gel91b}
Michael Gelfond and Vladimir Lifschitz.
\newblock Classical negation in logic programs and disjunctive databases.
\newblock {\em New Generation Computing}, 9:365--385, 1991.

\bibitem[\protect\citeauthoryear{Heljanko and Niemel{\"a}}{2001}]{hel01}
Keijo Heljanko and Ilkka Niemel{\"a}.
\newblock Answer set programming and bounded model checking.
\newblock In {\em Working Notes of the AAAI Spring Symposium on Answer Set
  Programming}, 2001.

\bibitem[\protect\citeauthoryear{Hietalahti \bgroup \em et al.\egroup
  }{2000}]{hie00}
Maarit Hietalahti, Fabio Massacci, and Nielel{\"a} Ilkka.
\newblock a challenge problem for nonmonotonic reasoning systems.
\newblock In {\em Proceedings of the 8th International Workshop on
  Non-Monotonic Reasoning}, 2000.

\bibitem[\protect\citeauthoryear{Koksal \bgroup \em et al.\egroup
  }{2001}]{kok01}
Pinar Koksal, Kesim Cicekli, and I.~Hakki Toroslu.
\newblock Specification of wrokflow processes using the action description
  language $\cal c$.
\newblock In {\em Working Notes of the AAAI Spring Symposium on Answer Set
  Programming}, 2001.

\bibitem[\protect\citeauthoryear{Lifschitz and Turner}{1994}]{lif94e}
Vladimir Lifschitz and Hudson Turner.
\newblock Splitting a logic program.
\newblock In Pascal Van~Hentenryck, editor, {\em Proceedings of International
  Conference on Logic Programming (ICLP)}, pages 23--37, 1994.

\bibitem[\protect\citeauthoryear{Lifschitz \bgroup \em et al.\egroup
  }{1999}]{lif99d}
Vladimir Lifschitz, Lappoon~R. Tang, and Hudson Turner.
\newblock Nested expressions in logic programs.
\newblock {\em Annals of Mathematics and Artificial Intelligence}, 25:369--389,
  1999.

\bibitem[\protect\citeauthoryear{Lifschitz \bgroup \em et al.\egroup
  }{2001}]{lif01}
Vladimir Lifschitz, David Pearce, and Agustin Valverde.
\newblock Strongly equivalent logic programs.
\newblock {\em ACM Transactions on Computational Logic}, 2:526--541, 2001.

\bibitem[\protect\citeauthoryear{Lifschitz}{1996}]{lif96b}
Vladimir Lifschitz.
\newblock Foundations of logic programming.
\newblock In Gerhard Brewka, editor, {\em Principles of Knowledge
  Representation}, pages 69--128. CSLI Publications, 1996.

\bibitem[\protect\citeauthoryear{Lifschitz}{1999}]{lif99c}
Vladimir Lifschitz.
\newblock Answer set planning.
\newblock In {\em Proc.~ICLP-99}, pages 23--37, 1999.

\bibitem[\protect\citeauthoryear{Liu \bgroup \em et al.\egroup }{1998}]{liu98}
Xinxin. Liu, C.~R. Ramakrishnan, and Scott~A. Smolka.
\newblock Fully local and efficient evaluation of alternating fixed points.
\newblock In {\em Proc.~Fourth Int'l Conference on Tools and Algorithms for the
  Construction and Analysis of Systems}, pages 5--19, 1998.

\bibitem[\protect\citeauthoryear{Marek and Truszczy\'nski}{1991}]{mar91}
Victor Marek and Miros{\l}aw Truszczy\'nski.
\newblock Autoepistemic logic.
\newblock {\em Journal of ACM}, 38:588--619, 1991.

\bibitem[\protect\citeauthoryear{Marek and Truszczy\'nski}{1999}]{mar99}
Victor Marek and Miros{\l}aw Truszczy\'nski.
\newblock Stable models and an alternative logic programming paradigm.
\newblock In {\em The Logic Programming Paradigm: a 25-Year Perspective}, pages
  375--398. Springer Verlag, 1999.

\bibitem[\protect\citeauthoryear{Niemel{\"a} and Simons}{2000}]{nie00}
Ilkka Niemel{\"a} and Patrik Simons.
\newblock Extending the {Smodels} system with cardinality and weight
  constraints.
\newblock In Jack Minker, editor, {\em Logic-Based Artificial Intelligence},
  pages 491--521. Kluwer, 2000.

\bibitem[\protect\citeauthoryear{Niemel{\"a}}{1999}]{nie99}
Ilkka Niemel{\"a}.
\newblock Logic programs with stable model semantics as a constraint
  programming paradigm.
\newblock {\em Annals of Mathematics and Artificial Intelligence}, 25:241--273,
  1999.

\bibitem[\protect\citeauthoryear{Osorio \bgroup \em et al.\egroup
  }{2004}]{oso04}
Mauricio Osorio, Juan~Antonio Navarro, and Jos{\'e} Arrazola.
\newblock Safe beliefs for propositional theories.
\newblock Accepted to appear at Annals of Pure and Applied Logic, 2004.

\bibitem[\protect\citeauthoryear{Pearce \bgroup \em et al.\egroup
  }{2001}]{pea01}
David Pearce, Hans Tompits, and Stefan Woltran.
\newblock Encodings for equilibrium logic and logic programs with nested
  expressions.
\newblock In {\em Proceedings of Portuguese Conference on Artificial
  Intelligence (EPIA)}, pages 306--320, 2001.

\bibitem[\protect\citeauthoryear{Pearce}{1997}]{pea97}
David Pearce.
\newblock A new logical characterization of stable models and answer sets.
\newblock In J{\"u}rgen Dix, Luis Pereira, and Teodor Przymusinski, editors,
  {\em Non-Monotonic Extensions of Logic Programming (Lecture Notes in
  Artificial Intelligence 1216)}, pages 57--70. Springer-Verlag, 1997.

\bibitem[\protect\citeauthoryear{Pearce}{1999}]{pea99}
David Pearce.
\newblock From here to there: Stable negation in logic programming.
\newblock In D.~Gabbay and H.~Wansing, editors, {\em What Is Negation?} Kluwer,
  1999.

\bibitem[\protect\citeauthoryear{Pelov \bgroup \em et al.\egroup
  }{2003}]{pel03}
Nikolay Pelov, Marc Denecker, and Maurice Bruynooghe.
\newblock Translation of aggregate programs to normal logic programs.
\newblock In {\em Proc. Answer Set Programming}, 2003.

\bibitem[\protect\citeauthoryear{Soininen and Niemel{\"a}}{1998}]{soi98}
Timo Soininen and Ilkka Niemel{\"a}.
\newblock Developing a declarative rule language for applications in product
  configuration.
\newblock In Gopal Gupta, editor, {\em Proceedings of International Symposium
  on Practical Aspects of Declarative Languages (PADL)}, pages 305--319.
  Springer-Verlag, 1998.

\bibitem[\protect\citeauthoryear{Son and Lobo}{2001}]{son01}
Tran~Cao Son and Jorge Lobo.
\newblock Reasoning about policies using logic programs.
\newblock In {\em Working Notes of the AAAI Spring Symposium on Answer Set
  Programming}, 2001.

\bibitem[\protect\citeauthoryear{Son \bgroup \em et al.\egroup }{2007}]{son07}
Tran~Cao Son, Enrico Pontelli, and Phan~Huy Tu.
\newblock Answer sets for logic programs with arbitrary abstract constraint
  atoms.
\newblock {\em J. Artif. Intell. Res. (JAIR)}, 29:353--389, 2007.

\bibitem[\protect\citeauthoryear{Trajcevski \bgroup \em et al.\egroup
  }{2000}]{tra00}
Goce Trajcevski, Chitta Baral, and Jorge Lobo.
\newblock Formalizing (and reasoning about) the specifications of workflows.
\newblock In {\em Proceedings of the Fifth IFCIS International conference on
  Cooperative Information Systems (CoopIS'2000)}, 2000.

\bibitem[\protect\citeauthoryear{Turner}{2003}]{tur03}
Hudson Turner.
\newblock Strong equivalence made easy: nested expressions and weight
  constraints.
\newblock {\em Theory and Practice of Logic Programming}, 3(4,5):609--622,
  2003.

\end{thebibliography}
\bibliographystyle{named}

\end{document}